\documentclass{article}

\usepackage{journal}
\usepackage{natbib}
\usepackage[mathscr]{eucal}
\usepackage[dvipsnames]{xcolor}
\usepackage[backref,pageanchor=true,plainpages=false,pdfpagelabels,bookmarks,bookmarksnumbered,breaklinks,pdfborder={0 0 0},colorlinks=true,
	citecolor=Sepia,
	linkcolor=Sepia,
	filecolor=Sepia,
	urlcolor=Sepia
]{hyperref}

\usepackage{framed}
\usepackage{dirtytalk}
\usepackage[utf8]{inputenc} 
\usepackage[T1]{fontenc}    
\usepackage{booktabs}       
\usepackage{nicefrac}       
\usepackage{microtype}      
\usepackage[framemethod=TikZ]{mdframed}
\mdfsetup{skipabove=\topskip,skipbelow=\topskip}
\usepackage{times}

\usepackage{bm} 
\usepackage{dsfont} 
\usepackage{xspace}
\usepackage[algo2e,ruled,linesnumbered,vlined]{algorithm2e}
\usepackage[most]{tcolorbox}
\usepackage[normalem]{ulem}

\usepackage{amsthm}
\usepackage{mathtools}
\usepackage[shortlabels]{enumitem}

\makeatletter
\def\set@curr@file#1{\def\@curr@file{#1}} 
\makeatother
\usepackage{siunitx} 

\newtheorem{thm}{Theorem}
\newtheorem{lem}[thm]{Lemma}
\newtheorem{rem}{Remark}

\newtheorem{claim}[thm]{Claim}
\newtheorem{cor}[thm]{Corollary}

\newtheorem{defn}{Definition}
\newtheorem*{defn*}{Definition}

\usepackage{prettyref}
\newrefformat{fig}{Figure~\ref{#1}}
\newrefformat{alg}{Algorithm~\ref{#1}}
\newrefformat{def}{Definition~\ref{#1}}
\newrefformat{eqn}{Equation~\ref{#1}}
\newrefformat{cor}{Corollary~\ref{#1}}
\newrefformat{app}{Appendix~\ref{#1}}
\newrefformat{rem}{Remark~\ref{#1}}
\newrefformat{assu}{Assumption~\ref{#1}}
\newrefformat{q}{Question~\ref{#1}}
\newrefformat{clm}{Claim~\ref{#1}}

\jmlrheading{}{2025}{}{}{}{O. Montasser, A. Shetty, \& N. Zhivotovskiy}
\ShortHeadings{Montasser Shetty Zhivotovskiy}{Margin and Smoothed Benchmarks}

\renewcommand{\ln}{\log}

\newcommand{\inbrace}[1]{\left \{ #1 \right \}}
\newcommand{\inparen}[1]{\left ( #1 \right )}
\newcommand{\insquare}[1]{\left [ #1 \right ]}
\newcommand{\inangle}[1]{\left \langle #1 \right \rangle}

\newcommand{\norm}[2]{\lVert #2 \rVert_{#1}}

\newcommand{\abs}[1]{\left\lvert #1 \right\rvert}

\newlength{\dhatheight}

\newcommand{\floor}[1]{\left \lfloor #1 \right \rfloor}

\newcommand{\SET}[1]{\inbrace{#1}}

\DeclareMathOperator*{\sign}{sign}

\DeclareMathOperator*{\Ex}{\mathds{E}}

\DeclareMathOperator*{\Prob}{\mathds{P}}

\renewcommand{\Pr}{\Prob}

\newcommand{\eps}{\varepsilon}

\newcommand{\bbR}{{\mathbb R}}

\newcommand{\bbA}{{\mathbb A}}
\newcommand{\bbB}{{\mathbb B}}

\newcommand{\calC}{\mathcal{C}}

\newcommand{\calF}{\mathcal{F}}
\newcommand{\calG}{\mathcal{G}}
\newcommand{\calH}{\mathcal{H}}

\newcommand{\calN}{\mathcal{N}}

\newcommand{\calR}{\mathcal{R}}
\newcommand{\calS}{\mathcal{S}}

\newcommand{\calX}{\mathcal{X}}
\newcommand{\calY}{\mathcal{Y}}
\newcommand{\calZ}{\mathcal{Z}}

\newcommand{\ind}{\mathds{1}}
\newcommand{\Risk}{{\mathrm{R}}}

\newcommand{\vc}{\mathsf{vc}}
\newcommand{\fat}{\mathsf{fat}}
\newcommand{\lit}{\mathsf{lit}}

\newcommand{\OPT}{\mathsf{OPT}}

\newcommand{\Cover}{\mathsf{C}}
\newcommand{\Pack}{\mathsf{P}}
\newcommand{\pOPT}{\mathsf{OPT^{\gamma}_{pert}}}
\newcommand{\gOPT}{\mathsf{OPT^{\sigma,\eps}_{gauss}}}
\newcommand{\mOPT}{\mathsf{OPT^{\gamma}_{margin}}}
\newcommand{\hOPT}{\mathsf{OPT^{\gamma}_{hinge}}}
    \newcommand{\tilgOPT}{ {\mathsf{\dot{OPT}^{\sigma}_{gauss}}}}
    \newcommand{\tiltilgOPT}{ \mathsf{{\ddot{OPT}^{\sigma}_{gauss}}}}

\newtcolorbox{boxbox}{
  colback=white,
  boxrule=1pt,
  enhanced,
  sharpish corners
}

\title{Beyond Worst-Case Online Classification:\\VC-Based Regret Bounds for Relaxed Benchmarks
}
\author{%
\name{Omar Montasser} \email{omar.montasser@yale.edu}\\
\addr Yale University\\
\name{Abhishek Shetty} \email{shetty@mit.edu}\\
\addr MIT\\
\name{Nikita Zhivotovskiy} \email{zhivotovskiy@berkeley.edu}\\
\addr UC Berkeley
}

\usepackage{charter}

\begin{document}

\maketitle

\begin{abstract}%
We revisit online binary classification by shifting the focus from competing with the best-in-class binary loss to competing against relaxed benchmarks that capture smoothed notions of optimality. Instead of measuring regret relative to the exact minimal binary error—a standard approach that leads to worst-case bounds tied to the Littlestone dimension—we consider comparing with predictors that are robust to small input perturbations, perform well under Gaussian smoothing, or maintain a prescribed output margin. Previous examples of this were primarily limited to the hinge loss. Our algorithms achieve regret guarantees that depend only on the VC dimension and the complexity of the instance space (e.g., metric entropy), and notably, they incur only an $O(\log(1/\gamma))$ dependence on the generalized margin $\gamma$. This stands in contrast to most existing regret bounds, which typically exhibit a polynomial dependence on $1/\gamma$. We complement this with matching lower bounds. Our analysis connects recent ideas from adversarial robustness and smoothed online learning.
\end{abstract}

\begin{keywords}
online learning, binary classification, generalized margin, regret bounds, VC dimension, Littlestone dimension, adversarial robustness, smoothed online learning
\end{keywords}

\section{Introduction}
We revisit the problem of online learning, specifically online binary classification,  which is arguably archetypical setting for sequential decision making, that much of the later theory is built upon.
It is well-known that a hypothesis class $\mathcal{H}$ is online learnable if and only if $\mathcal{H}$ has finite Littlestone dimension \citep*{DBLP:journals/ml/Littlestone87, DBLP:conf/colt/Ben-DavidPS09}. In particular, it is well-understood that minimizing regret relative to the smallest achievable error with a class $\mathcal{H}$ is quantified (up to constant factors) by the Littlestone dimension of $\mathcal{H}$, denoted $\lit(\mathcal{H})$, in both the realizable \citep{DBLP:journals/ml/Littlestone87} and agnostic cases \citep*{DBLP:conf/colt/Ben-DavidPS09, DBLP:conf/stoc/AlonBDMNY21}.

Though, we have this precise combinatorial characterization, online learning is challenging. 
This is exemplified by arguably the simplest hypothesis class: thresholds on the unit interval 
{
\setlength{\abovedisplayskip}{2pt}
\setlength{\belowdisplayskip}{2pt}
\[
\mathcal{H} = \Big\{x \mapsto \operatorname{sign}(x - \theta)\,\Big|\,\theta \in [0,1]\Big\}
\]
}
which is not online learnable since it has infinite Littlestone dimension, implying that any learner can be forced to make infinitely many mistakes even when the adversarial sequence is realizable by a threshold. 
Needless to say, the learning of classes induced by linear functions, such as thresholds or general halfspaces, is arguably one of the most basic problems in machine learning.

Given this pessimistic situation, the learning theory community has developed several techniques to bypass the above lower bound.
\begin{itemize}
    \item A classical perspective on learning linear classifiers, that perhaps even predated the modern theory of online learning, is the assumption of \emph{margin}. 
    It is well known that, when the online sequence satisfies a margin assumption, the Perceptron algorithm \citep{rosenblatt1958perceptron} can learn thresholds (and more generally halfspaces) with a mistake bound of $O(1/\gamma^2)$ \citep{novikoff1962convergence}, where $\gamma>0$ is the margin parameter. The sequential margin bound can be generalized to the agnostic case \citep{cesa2005second, DBLP:journals/corr/abs-1305-0208}, showing the same polynomial dependence on the inverse margin $1/\gamma$. 
    \item A more recent direction is \emph{smoothed online learning}. Simplifying the setup, the idea is to assume that there is some known base density $\mu$ such that, at each round, the new observation $X_t$ is generated from a density which has density ratio with respect to $\mu$ bounded by $1/\sigma$, where $\sigma > 0$ is called the smoothness parameter. 
    In its simplest form, this assumption allows one to prove regret bounds of the form $O\bigl(\sqrt{\vc(\calH) T \ln(T/\sigma)}\bigr)$\citep*{DBLP:journals/jacm/HaghtalabRS24}.
\end{itemize}

    In this work, with a similar aim of bypassing pessimistic lower bounds, we study online learning from a different perspective: \textit{relaxing the notion of optimality}.
    That is, instead of minimizing regret relative to the smallest achievable error with a class $\calH$, denoted by $\OPT$, we consider minimizing regret relative to relaxed benchmarks: $\pOPT$ (\ref{eqn:OPT_input}), $\gOPT$ (\ref{eqn:OPT_gauss}), and $\mOPT$ (\ref{eqn:OPT_output}).
    These  can be thought of as generalizations of the classic margin assumption for halfspaces that are defined more broadly for generic hypothesis classes. The introduction of these benchmarks is partially inspired by the seminal work of \citet{DBLP:journals/jacm/SpielmanT04} on smoothed analysis, and more recently the work of \citet*{DBLP:conf/colt/ChandrasekaranK24} which demonstrated the computational benefits of competing with a \textit{relaxed} notion of optimality in agnostic PAC learning.
    This work explores \emph{statistical} benefits of these relaxations in the context of online learning.

    To better understand our motivation, we put some existing results in context. Arguably the most well-known relaxation of the binary loss in the sequential setting, closely related to the margin loss, is the \emph{hinge loss}, whose normalized version for $y \in \{-1, 1\}$ and $f(x)$ a real-valued predictor satisfies
    \[
    \ind[\operatorname{sign}(f(x)) \neq y] \le \frac{\max\SET{0, \gamma - yf(x)}}{\gamma}.
    \]
    The following regret bound, relevant to our discussion, is given by the Perceptron algorithm \citep[see e.g., Corollary 1 in][]{DBLP:journals/corr/abs-1305-0208}; see also \citep{cesa2005second}). For any $\gamma>0$, and any (adversarially chosen) sequence $(x_t, y_t)_{t = 1}^T$ with $y_t \in \{-1, 1\}$ and $x_t \in \mathbb{R}^d, \norm{2}{x_t}\leq 1$, 
    {
    \setlength{\abovedisplayskip}{2pt}
    \setlength{\belowdisplayskip}{2pt}
    \begin{equation}
        \label{eq:perceptron-hinge}
        \sum\limits_{t = 1}^T\ind[\widehat{y}_t\neq y_t] - \underbrace{\min_{w \in \bbR^d, \norm{2}{w}=1 }\sum\limits_{t = 1}^T\frac{\max\SET{0, \gamma - y_t\inangle{w,x_t}}}{\gamma}}_{\hOPT} \leq \sqrt{\frac{T}{\gamma^2}}.
    \end{equation}
    }
In the particular setup of the regret bound \eqref{eq:perceptron-hinge}, we observe polynomial dependence on \(\tfrac{1}{\gamma}\) (as in Novikoff's margin bound). However, it can be shown that requiring polynomial dependence on \(\tfrac{1}{\gamma}\) is overly pessimistic. 
For context, the regret bound of \citet*{DBLP:conf/colt/Gilad-BachrachNT04}, 
in the margin setting of Novikoff, provides an \(O\bigl(d\ln\!\bigl(\tfrac{1}{\gamma}\bigr)\bigr)\) bound 
and thus achieves dimension dependence alongside a more favorable \emph{logarithmic} dependence on \(\tfrac{1}{\gamma}\). Recently, \citet*{qian2024refined} extended the bound \eqref{eq:perceptron-hinge} using a version of the exponential weights algorithm with respect to the hinge loss, again combining dependence on \(d\) 
with only logarithmic dependence on \(\tfrac{1}{\gamma}\). 

An important remark regarding the comparison of bounds \(O\bigl(d\ln\!\bigl(\tfrac{1}{\gamma}\bigr)\bigr)\) and Novikoff's \say{dimension-free} bound \(O\bigl(\!\tfrac{1}{\gamma^2}\bigr)\) is in order. While each has regimes where it is preferable, Novikoff’s bound relies on the rescaling \(\max_t\|x_t\|_2 \leq 1\), which is often unrealistic in high dimensions where norms typically grow as \(\sqrt{d}\) (e.g., for a multivariate Gaussian distribution). In the natural rescaling where \(\max_t\|x_t\|_2 \sim \sqrt{d}\), our parametric bounds scale as \(O\bigl(d\ln\!\bigl(\tfrac{d}{\gamma^2}\bigr)\bigr)\), usually outperforming Novikoff’s weaker \(O\bigl(\!\tfrac{d}{\gamma^2}\bigr)\) bound. This serves as additional motivation for studying \(O\bigl(d\ln\!\bigl(\tfrac{1}{\gamma}\bigr)\bigr)\)-style regret bounds. A more detailed discussion is deferred to \prettyref{sec:relatedwork}.

Since the hinge loss is merely one form of relaxing the binary loss, 
and noting the surprising lack of results in the literature with logarithmic dependence 
on the inverse generalized margin \(\tfrac{1}{\gamma}\), 
we are interested in understanding when such favorable regret bounds can be achieved in broader scenarios:
\begin{boxbox}
We aim for new regret bounds that replace the prohibitive Littlestone dimension 
with dependence on the VC dimension, 
while incurring only a logarithmic dependence on the inverse generalized margin \(\tfrac{1}{\gamma}\), 
by competing against one of the smoothed comparators \(\OPT^{\gamma}\) given below by \eqref{eqn:OPT_input}, \eqref{eqn:OPT_gauss}, \eqref{eqn:OPT_output}.
\end{boxbox}

\paragraph{Notation and Preliminaries.} We consider instance spaces $\calX$ that are equipped with a metric $\rho:\calX\times \calX \to \mathbb{R}_{\geq 0}$, and a label space $\calY =\SET{\pm 1}$. That is, in what follows, we assume that $y_t \in \{ \pm 1 \}$. Moreover, we assume that any class of classifiers $\mathcal{H}$ consists of mappings from $\mathcal{X}$ to $\{ \pm 1 \}$, and we denote by $\vc(\mathcal H)$ the VC dimension of $\calH$. We explicitly mention cases where we work with real-valued predictors, usually denoted by $\calF \subseteq [-1, +1]^\calX$. We denote by $\vc(\calF$) the pseudo-dimension, and $\fat_{\calF}(\tau)$ the fat-shattering dimension at scale $\tau$. We denote by $B(x, \gamma)=\SET{z\in \calX:\rho(x,z)\leq \gamma}$ a ball of radius $\gamma$ centered on $x$ relative to metric $\rho$. We denote by $\Cover(\calX, \rho, \gamma)$ a covering of $\calX$ with respect to metric $\rho$ at scale $\gamma$, and we denote by $\Pack(\calX, \rho, \gamma)$ a packing of $\calX$ with respect to metric $\rho$ at scale $\gamma$. It is well known that $|\Pack(\calX, \rho, 2\gamma)|\leq |\Cover(\calX, \rho, \gamma)| \leq |\Pack(\calX, \rho, \gamma)|$ \citep{kolmogorov1959varepsilon}. For an arbitrary norm $\norm{}{\cdot}$ on $\bbR^d$ and the unit ball $\calX=\SET{x\in \bbR^d: \norm{}{x}\leq 1}$, for $\gamma < 1$, it is well-known that $d\ln\inparen{\nicefrac{1}{\gamma}} \leq \ln |\Cover(\calX, \norm{}{\cdot}, \gamma)| \leq d\ln\inparen{1+\nicefrac{2}{\gamma}}$ \citep[e.g., Corollary 27.4 in][]{Polyanskiy_Wu_2025}. 
We denote by $\calN$ a standard multivariate Gaussian distribution $\calN(0, I_d)$, and by $\Phi^{-1}$ the inverse CDF of a univariate standard Gaussian. 

\section{Our Contributions}

As discussed above, instead of minimizing regret relative to the smallest achievable error with class $\calH$ where dependence on Littlestone dimension is unavoidable, we consider minimizing regret relative to \textit{relaxed} notions of optimality. These relaxed notions can be thought of as generalizations of the margin assumption in the special case of halfspaces. 

\subsection*{Main Result I: Competing with an Optimal Predictor under Worst-Case Perturbations.} 
We consider competing with the smallest achievable error with class $\calH$ under worst-case perturbations of $x_t$ of distance at most $\gamma$ away. To formalize this, we assume $\mathcal{X}$ is equipped with a metric $\rho$. 
Let $ B( x , \gamma ) = \left\{ z \in \mathcal{X}: \rho(x, z) \leq \gamma \right\} $ denote the ball of radius $\gamma$ centered at $x$ with respect to $\rho$.  
Define the following relaxed benchmark: 
{
\setlength{\abovedisplayskip}{2pt}
\setlength{\belowdisplayskip}{2pt}
\begin{equation}
\label{eqn:OPT_input}
    \pOPT \doteq \min_{h \in \calH} \sum_{t=1}^{T} \max_{z_t \in B(x_t, \gamma)} \ind\insquare{h(z_t) \neq y_t}.
\end{equation}
}
For an intuitive understanding of this benchmark, consider the realizable case where $\pOPT=0$. It means the adversarial online sequence $(x_t, y_t)_{t = 1}^T$ satisfies a ``margin'' assumption with respect to perturbations of $x_t$'s: there exists an $h^\star\in \calH$ that labels the entire $\gamma$-ball around each $x_t$ with $y_t$ for all $1\leq t \leq T$. For example, in the special case of halfspaces, this assumption is equivalent to the classical margin assumption (see \prettyref{clm:halfspaces-equivalence} and \prettyref{lem:perturbation-halfspace}). More generally, in the agnostic case, we compete with $\pOPT>0$ without any assumptions on the adversarial online sequence $(x_t, y_t)_{t = 1}^T$.

\begin{rem}Observe that when $\gamma = 0$, $\pOPT$ reduces to the standard  binary $\OPT$ in online learning. In fact, our relaxed benchmarks $\gOPT$ and $\mOPT$ (introduced below) also converge to $\OPT$, as $\gamma$, $\eps$, and $\sigma$ approach $0$. Thus, our goals are: (1) to get the best possible dependence on $\gamma$, $\eps$, and $\sigma$ in regret bounds, and (2) to eliminate dependence on the Littlestone dimension.
\end{rem}

Our first main result is an online learning algorithm with a regret guarantee relative to $\pOPT$ that depends on the VC dimension of $\calH$, bypassing dependence on the Littlestone dimension of $\calH$. 

\begin{boxbox}
    \begin{center}
        \textbf{Main Result I (\prettyref{thm:input-margin-upperbnd} and \prettyref{thm:input-margin-lowerbnd})} 
        \vspace{-0.4cm}
    \end{center}
    For any metric space $(\calX,\rho)$, any $\gamma > 0$, and any class $\calH\subseteq \calY^\calX$, \prettyref{alg:input-margin} guarantees for any sequence $(x_1,y_1),\dots, (x_T, y_T)$, an expected number of mistakes of
    {
    \setlength{\abovedisplayskip}{2pt}
    \setlength{\belowdisplayskip}{2pt}
    \[\sum_{t=1}^{T} \Ex\ind[ \hat{y}_t \neq y_t ] - \pOPT \leq  \sqrt{T\cdot \vc(\calH)\ln\inparen{\frac{e\abs{\Cover(\calX,\rho,\gamma)}}{\vc(\calH)}}}.\]
    }
    Furthermore, for any metric space $(\calX,\rho)$, there is a class $\calH$ where this bound is tight. 
\end{boxbox}
The upper bound depends on both the VC dimension of $\mathcal{H}$ and the metric entropy of $\mathcal{X}$, which, intuitively, can lead to a quadratic dependence on the {dimension} of $\mathcal{X}$ (e.g., when $\mathcal{X} \subseteq \mathbb{R}^d$) under the square root. Indeed, as shown in Theorem \ref{thm:halfspaces-upperbound}, this dependence is suboptimal for classes induced by halfspaces. Nevertheless, the key insight of the above result is the matching lower bound, which demonstrates that for certain function classes, both the metric entropy of $\mathcal{X}$ and the VC dimension of $\mathcal{H}$ must be taken into account. To be more specific, for an arbitrary norm $\norm{}{\cdot}$ on $\bbR^d$ and the unit ball $\calX=\SET{x\in \bbR^d: \norm{}{x}\leq 1}$, for $\gamma < 1$, it is well-known that $\ln(|\Cover(\calX, \norm{}{\cdot}, \gamma)|) \leq d\ln\inparen{1+\nicefrac{2}{\gamma}}$ \citep[e.g., Corollary 27.4 in][]{Polyanskiy_Wu_2025}. Hence, \prettyref{thm:input-margin-upperbnd} implies the following corollary,
{
\setlength{\abovedisplayskip}{2pt}
\setlength{\belowdisplayskip}{2pt}
\[\sum_{t=1}^{T} \Ex\ind[ \hat{y}_t \neq y_t ] - \pOPT \lesssim  \sqrt{T\cdot \vc(\calH)\cdot d \cdot \ln\inparen{1+\nicefrac{2}{\gamma}}}.\]}
So, it is natural to ask whether it is possible avoid dependence on the dimension $d$ of $\calX$. But, our lower bound shows that it is not possible to avoid dependence on the metric entropy of $\calX$, $\ln |\Cover(\calX, \norm{}{\cdot}, \gamma)|$, which implies that dependence on dimension $d$ (or its analogs) of $\calX$ is unavoidable in general.

We further note that the benchmark considered here is closely related to the smoothed online learning perspective on beyond worst-case analysis of online learning (discussed in further detail in \prettyref{sec:smoothed_online}). 
In fact, a slightly more general result can be derived by using the machinery of smoothed online learning (which we present as \prettyref{cor:metric}). 
At a fundamental level, both these results rely on similar approximations of the metric space and the function class but we present \prettyref{thm:input-margin-upperbnd} as a more direct approach which allows a more straightforward comparison to bounds considered in the literature on margin and robustness. 

\subsection*{Main Result II: Competing with a Gaussian-Smoothed Optimal Predictor.}
We now consider the setup where $\mathcal X \subseteq \mathbb{R}^d$ and we compete with a different relaxation: the smallest achievable error with class $\calH$ under random perturbations of $x_t$ drawn from a multivariate Gaussian distribution $\calN(0,\sigma^2 I_d)$. Formally,
{\setlength{\abovedisplayskip}{2pt}
\setlength{\belowdisplayskip}{2pt}
\begin{equation}
    \label{eqn:OPT_gauss}
    \gOPT \doteq \min_{h \in \calH} \sum_{t=1}^{T} \ind\insquare{ y_t \cdot \Ex_{z\sim \calN(0, I_d)}\insquare{h(x_t+\sigma z)} \leq \eps}.
\end{equation}
}
In words, we are competing with the best predictor $h^\star\in \calH$ that minimizes the number of rounds $t$ for which the fraction of wrongly-classified Gaussian perturbations, $\Prob_{z\sim \calN}\SET{h^\star(x_t+\sigma z) \neq y_t}$, exceeds the threshold of $\nicefrac{1}{2} - \nicefrac{\eps}{2}$. Compared with $\pOPT$ \eqref{eqn:OPT_input}, instead of minimizing error against \textit{worst-case} perturbations of radius $\gamma$, here we just require the probability of error under \emph{random} Gaussian perturbations to be slightly smaller than $\nicefrac{1}{2}$. The realizable case where $\gOPT=0$ means the adversarial online sequence $(x_t,y_t)_{t=1}^{T}$ satisfies a ``margin'' assumption relative to Gaussian perturbations: there exists an $h^\star \in \calH$ that labels more than $\nicefrac{1}{2}+\nicefrac{\eps}{2}$ of the Gaussian perturbations $x_t + \sigma z$ with the label $y_t$ for all $1\leq t \leq T$. For example, in the special case of halfspaces, we show that this is equivalent to the classical margin assumption (see \prettyref{clm:halfspaces-equivalence} and \prettyref{lem:gaussian-margin-rel}).

Our main result is an online learning algorithm with a regret guarantee relative to $\gOPT$ that depends on the VC dimension of $\calH$, bypassing dependence on the Littlestone dimension of $\calH$.

\begin{boxbox}
    \begin{center}
        \textbf{Main Result II (\prettyref{thm:gauss-upperbnd} and \prettyref{thm:gauss-lowerbnd})} 
        \vspace{-0.4cm}
    \end{center}
    For any $\calX\subseteq \bbR^d$, any $\sigma, \eps > 0$, for any class $\calH\subseteq \calY^{\bbR^d}$, \prettyref{alg:gaussian} guarantees for any sequence $(x_1,y_1),\dots, (x_T, y_T)$, an expected number of mistakes of
    {\setlength{\abovedisplayskip}{2pt}
    \setlength{\belowdisplayskip}{2pt}
    \[\sum_{t=1}^{T} \Ex\ind[ \hat{y}_t \neq y_t ] - \gOPT \lesssim \sqrt{T\cdot \vc(\calH)\ln\inparen{\frac{\abs{\Cover\inparen{\calX,\lVert \cdot \rVert_2, \sqrt{\pi/32}\cdot\sigma \eps}}}{\eps^2}}}.\]
    }
    Furthermore, for $\calX=[0,1]$, there is a class $\calH$ where this bound is tight (up to log factors).
\end{boxbox}

It follows as a corollary that for the Euclidean unit-ball $\calX=\SET{x\in \bbR^d: \norm{2}{x}\leq 1}$, we can achieve a regret 
$\sum_{t=1}^{T} \Ex\ind[ \hat{y}_t \neq y_t ] - \gOPT \lesssim \sqrt{T\cdot \vc(\calH) \cdot d \cdot \ln\inparen{\frac{1}{\eps \sigma}}}$. A natural question to ask here is whether it is possible to take $\eps = 0$. To this end, our lower bound implies that the dependence on $\ln\inparen{\frac{1}{\eps \sigma}}$ is necessary, since the covering number $|\Cover([0,1],\abs{\cdot}, 4\sigma\eps)|=\Omega\inparen{\frac{1}{\eps \sigma}}$, and therefore it is impossible to compete with $\gOPT$ with $\sigma=0$ or $\eps=0$. 

We note that regret bounds closely related to this benchmark can be achieved using a smoothed online learning perspective, as discussed in \prettyref{sec:smoothed_online_gaussian}.
The algorithms from smoothed online learning can be used to compete with the benchmark, referred to as $\tiltilgOPT$,
{
\setlength{\abovedisplayskip}{2pt}
\setlength{\belowdisplayskip}{2pt}
\begin{align}
    \tiltilgOPT \doteq  \min_{h \in \mathcal{H}} \sum_{t=1}^{T} \Pr_{z_t \sim \mathcal{N}(0,I_d)} \left[ h(x_t + \sigma z_t) \neq y_t \right].
\end{align}
}
Using techniques from smoothed online learning, we can achieve a regret bound\footnote{For technical reasons, the formal regret bound requires replacing the volume of $\mathcal{X}$ with the volume of a dilation.} of 
\begin{align}
     \sum_{t=1}^{T} \Ex \ind[ \hat{y}_t \neq y_t ]  -  \tiltilgOPT \leq  \sqrt{T\cdot \vc(\calH)\ln\inparen{  \frac{\mathrm{Vol}( \mathcal{X}    )}{  (\sqrt{ 2 \pi \sigma^2 })^d  }      } }. 
\end{align}
A more formal discussion of this benchmark and technique is deferred to \prettyref{sec:smoothed_online}.

Though, at a fundamental level, the reasoning behind both benchmarks are similar, $\gOPT$ provides a more refined bound since $ \gOPT \leq (2+o(1))\tiltilgOPT$ by choosing $\eps=1/(T^2)$ (\prettyref{clm:gaussian}), 
and no general reverse inequality is true. In fact, it is possible to construct a situation where $\tiltilgOPT=\Omega(T)$ and $\gOPT=0$.\footnote{On a technical note, smoothed online learning can compete against sharper benchmark, $\tilgOPT$, and cannot be compared directly to $\gOPT$.  
The relation between these benchmarks are discussed further in \prettyref{sec:smoothed_online}.}  

\subsection*{Main Result III: Competing with an Optimal Predictor with a Margin.}
For a (real-valued) class $\calF \subseteq [-1, +1]^\calX$, we consider competing with the smallest achievable error when restricting to functions $f\in \calF$ that have an output margin of $\gamma$. Formally,
{
\setlength{\abovedisplayskip}{2pt}
\setlength{\belowdisplayskip}{2pt}
\begin{equation}
    \label{eqn:OPT_output}
    \mOPT \doteq \min_{f\in \calF} \sum_{t=1}^{T} \ind[y_t f(x_t) \leq \gamma].
\end{equation}
}
Competing with $\mOPT$ was studied in the literature before. \citet*{DBLP:conf/colt/Ben-DavidPS09} showed that in general minimizing regret relative to $\mOPT$ is characterized by a natural extension of the classical Littlestone dimension that considers the $\gamma$-margin loss $(x,y)\mapsto \ind[yf(x)\leq \gamma]$, and \citet*{DBLP:conf/nips/RakhlinST10, DBLP:journals/jmlr/RakhlinST15} gave a non-constructive online learner achieving a regret bound of $O(\calR_T(\calF)/\gamma)$ where $\calR_T(\calF)$ denotes the (unnormalized) sequential Rademacher complexity (ignoring other mild additive terms). 

We show next that under an additional Lipschitzness assumption on the class $\calF$, it is possible to achieve regret relative to $\mOPT$ that depends on the minimum of the pseudo-dimension of $\calF$ and the fat-shattering dimension of $\calF$, bypassing dependence on the sequential Rademacher complexity, and with only a logarithmic dependence on $1/\gamma$ and $L$. We complement this with a lower bound showing that dependence on $1/\gamma$ and $L$ is unavoidable in general, and therefore showing that the Lipschitzness assumption on $\calF$ is \emph{necessary} to achieve VC-based regret guarantees. 

\begin{boxbox}
    \begin{center}
        \textbf{Main Result III (\prettyref{thm:output-margin-upperbnd} and \prettyref{thm:output-margin-lowerbnd})} 
        \vspace{-0.4cm}
    \end{center}
For any metric space $(\calX,\rho)$, any $\gamma > 0$, and any function class $\calF\subseteq [-1,1]^\calX$ that is $L$-Lipschitz relative to $\rho$, there exists an online learner such that for any sequence $(x_1,y_1),\dots, (x_T, y_T)$, the expected number of mistakes satisfies
{
\setlength{\abovedisplayskip}{2pt}
\setlength{\belowdisplayskip}{2pt}
    \[\sum_{t=1}^{T} \Ex\ind[ \hat{y}_t \neq y_t ] - \mOPT \lesssim \sqrt{T\cdot \min\SET{G_0, G_{\gamma/4}}},\]
}
where $G_0 \leq {\vc(\calF)}\ln\inparen{\frac{e\abs{\Cover(\calX,\rho,\nicefrac{\gamma}{2L})}}{\vc(\calF)}}$ , and for any $\alpha >0$, there are constants $c_1, c_2, c_3 > 0$ such that $G_{\gamma/4} \leq c_1\fat_\calF\inparen{c_2\alpha \frac{\gamma}{4}} \ln^{1+\alpha}\inparen{\frac{c_3 \abs{\Cover(\calX,\rho,\nicefrac{\gamma}{2L})}}{\fat_\calF\inparen{c_2\frac{\gamma}{4}}\cdot\frac{\gamma}{4}}}$.

Furthermore, for the space $\calX=[0,1]$, there is a class $\calF$ where this bound is tight.
\end{boxbox}

\subsection*{Main Result IV: Halfspaces.}
For generic hypothesis classes, the benchmarks $\pOPT$ \eqref{eqn:OPT_input}, $\gOPT$ \eqref{eqn:OPT_gauss}, and $\mOPT$ \eqref{eqn:OPT_output} are incomparable. For example, in the realizable case they represent different assumptions on the adversarial online sequence $(x_t, y_t)_{t=1}^{T}$. But, for halfspaces, these benchmarks are \emph{equivalent}.
\begin{claim}
\label{clm:halfspaces-equivalence}
For (homogeneous) halfspaces $\mathcal{H}=\SET{x \mapsto \sign\inparen{\inangle{w,x}}: w \in \mathbb{R}^d}$, there is an equivalence between competing with the three introduced relaxed benchmarks: $\pOPT$ \eqref{eqn:OPT_input}, $\gOPT$ \eqref{eqn:OPT_gauss}, and $\mOPT$ \eqref{eqn:OPT_output}. In particular, $\mOPT = \pOPT$ for all $\gamma > 0$, and $\mOPT = \gOPT$ for all $\eps,\sigma, \gamma > 0$ satisfying $\gamma=\sigma\Phi^{-1}\inparen{\nicefrac{1}{2} +\nicefrac\eps2}$.
\end{claim}

The proof of \prettyref{clm:halfspaces-equivalence} is deferred to \prettyref{sec:halfspaces}. We show next that it is possible to compete with these relaxed benchmarks, with a regret bound of $O\inparen{\sqrt{Td\ln(1/\gamma)}}$. This generalizes results from the literature which considered the $\ell_2$-norm and the realizable case \citep*{DBLP:conf/colt/Gilad-BachrachNT04, rakhlin2014statistical}, to handle arbitrary norms and the agnostic case. We also note that our earlier generic result (\prettyref{thm:input-margin-upperbnd}) implies a regret bound of $O\inparen{\sqrt{Td^2\ln(1/\gamma)}}$ for halfspaces (which is unavoidable for generic classes), but our result below bypasses this by utilizing the parametric structure of halfspaces (see \prettyref{sec:halfspaces} for further details). 
\begin{boxbox}
    \begin{center}
        \textbf{Main Result IV (\prettyref{thm:halfspaces-upperbound})} 
        \vspace{-0.4cm}
    \end{center}
    For any normed vector space $(\calX,\norm{}{\cdot})$ where $\mathcal{X}\subseteq \mathbb{R}^d$ and $B=\sup_{x\in\calX} \|x\| < \infty$, and any $\gamma > 0$, there is an online learner such that for any sequence $(x_1,y_1),\dots, (x_T, y_T)$, the  expected number of mistakes satisfies
    \[\sum_{t=1}^{T} \Ex\ind[ \hat{y}_t \neq y_t ] - \min_{w\in \bbR^d, \norm{\star}{w}=1} \sum_{t=1}^{T} \ind\insquare{y_t \inangle{w, x_t} \leq \gamma} \leq \sqrt{T\cdot  d\ln\inparen{1+\frac{2 B}{\gamma}}}.\]
\end{boxbox}

\section{Discussion and Related Work}
\label{sec:relatedwork}

First, we remark that using a standard trick of running Multiplicative Weights with a prior over a suitable discretization of the parameters $\gamma, \sigma, \eps$ (representing different instantiations of our online learning algorithms) \citep*[see e.g.,][]{DBLP:journals/jmlr/RakhlinST15}, we can achieve an even stronger regret guarantees of the form $\inf_{\gamma>0} \{\OPT^{\gamma} + \sqrt{\cdots}\}$ at the expense of an additional term that is doubly-logarithmic in $\gamma, \sigma, \eps$. We also remark that because our algorithms are based on Multiplicative Weights (see \prettyref{lem:multiplicativeweights}), we immediately obtain improved first order regret bounds of the form $\sqrt{2\OPT^\gamma\cdot \blacksquare} + \blacksquare$, instead of $\sqrt{T\cdot \blacksquare}$ in all of our results.

\paragraph{Computational Efficiency.} Our algorithms are based on constructing appropriate covers $\calC$ of $\calH$ and then running Multiplicative Weights with $\calC$ as experts. Investigating computationally efficient versions of our proposed algorithms is an interesting direction to explore in depth in future work. For now, we emphasize that there is a limited number of results in this direction in the context of $d\log(1/\gamma)$-style regret bounds. For example, cutting plane methods can be used for halfspaces in the realizable setting (i.e., when $\mOPT, \pOPT, \gOPT = 0$)\citep*{DBLP:conf/colt/Gilad-BachrachNT04}. Another positive result due to \citet*{qian2024refined} is a $\sqrt{Td\log(1/\gamma)}$-type regret bound competing with the smallest achievable hinge loss $\OPT^{\gamma}_{\rm hinge}$ \eqref{eq:perceptron-hinge}, which can be implemented in polynomial time with efficient unconstrained sampling from log-concave measures. On the other hand, competing with $\mOPT$ appears to be more challenging computationally, where the best known algorithms for halfspaces (in the PAC setting) incur a runtime that is exponential in $1/\gamma$ \citep*{shalev2009agnostically,DBLP:conf/nips/BirnbaumS12,DBLP:conf/nips/DiakonikolasKM19}, which perhaps suggests that we should not expect efficient algorithms in the online setting.  

\paragraph{Partial Concept Classes.} Our generic algorithmic upper bounds for minimizing regret relative to the relaxed benchmarks: $\pOPT$ (\ref{eqn:OPT_input}), $\gOPT$ (\ref{eqn:OPT_gauss}), and $\mOPT$ (\ref{eqn:OPT_output}), bypass dependence on the Littlestone dimension, and instead depend on the VC dimension and metric entropy. Our lower bounds also exhibit examples of classes where it is not possible to improve on these regret bounds. But, it is natural to ask whether there exists a generic online learning algorithm that is optimal for all hypothesis classes $\calH$ and to characterize the optimal regret. To this end, we remark that we can answer this via the language of \emph{partial concept classes} \citep*{DBLP:conf/focs/AlonHHM21}. A partial concept class $\calH\subseteq \SET{-1,1,\star}^\calX$ is a collection of functions where each $h\in \calH$ is a partial function $h: \calX\to \SET{-1,1,\star}$ and $h(x)=\star$ indicates that $h$ is undefined at $x$. The classic Littlestone dimension naturally extends to partial concept classes without modification, and continues to characterize online learnability of partial concept classes \citep{DBLP:conf/focs/AlonHHM21}. We note that our results can be viewed as online learning guarantees for the following generic partial concept classes:
\begin{itemize}[itemsep=0pt]
    \item Competing with $\pOPT$ is equivalent to online learning the partial concept class $\calH_{\gamma}= \SET{h_{\gamma}\mid h\in \calH}$ where $h_\gamma(x) = y \text{ if } \forall z\in B(x,\gamma),\, h(z)=y,\text{ and } h_\gamma(x)=\star \text{ otherwise.}$
    \item Competing with $\gOPT$ is equivalent to online learning the partial concept class $\calH_{\sigma,\eps}= \SET{h_{\sigma, \eps}\mid h\in \calH}$ where $h_{\sigma,\eps}(x) = y \text{ if } y\Ex_{z\sim \calN}\insquare{h(x+\sigma z)} > \eps,\text{ and } h_{\sigma,\eps}(x)=\star \text{ otherwise.}$
    \item Competing with $\mOPT$ is equivalent to online learning the partial concept class $\calF_\gamma = \SET{f_\gamma \mid f\in \calF}$ where $f_{\gamma}(x) = y \text{ if } yf(x) > \gamma,\text{ and } f_{\gamma}(x)=\star \text{ otherwise.}$
\end{itemize}
For optimal regret, we can run the Standard Optimal Algorithm \citep{DBLP:journals/ml/Littlestone87} using the partial concept classes defined above. By itself, this observation is not very insightful as the regret will be characterized in terms of the Littlestone dimension of the partial concept class $(\calH_{\gamma}, \calH_{\sigma,\eps}, \calF_\gamma)$ and a-priori it is unclear whether these quantities can be bounded by the VC dimension and metric entropy. But, combined with our upper bounds (Theorems~\ref{thm:input-margin-upperbnd}, \ref{thm:gauss-upperbnd}, and \ref{thm:output-margin-upperbnd}), we immediately get as a corollary that the Littlestone dimension of there partial classes is bounded in terms of the VC dimension and metric entropy.

Another potentially interesting connection is with \emph{differentially private} PAC learning. It is known that a class $\calH$ is differentially privately PAC learnable if and only if $\calH$ has finite Littlestone dimension \citep*{DBLP:journals/jacm/AlonBLMM22}. For partial concept classes, it remains open whether finite Littlestone dimension implies differentially private PAC learning \citep*{DBLP:conf/focs/FioravantiHMST24}. If this question is resolved positively, then combined with our results it would imply that the partial concept classes $\calH_{\gamma}, \calH_{\sigma,\eps}, \calF_\gamma$ discussed above are differentially privately PAC learnable. Such a (potential) result can be viewed as further benefits of studying relaxed benchmarks in learning theory, as it would allows us to bypass the worst-case dependence on Littlestone dimension in differentially private PAC learning.

\paragraph{Generic Margin Regret Bounds.}  
Competing with the relaxed benchmark of $\mOPT$ \eqref{eqn:OPT_output} was studied in the literature before. \citet*{DBLP:conf/colt/Ben-DavidPS09} showed that minimizing regret relative to $\mOPT$ is characterized by a natural extension of the classical Littlestone dimension which considers the $\gamma$-margin loss $(x,y)\mapsto \ind[yf(x)\leq \gamma]$, and \citet*{DBLP:conf/nips/RakhlinST10, DBLP:journals/jmlr/RakhlinST15} gave a non-constructive regret bound of $O\inparen{\frac{\calR_T(\calF)}{\gamma} + \sqrt{T}\inparen{3+\log\log\inparen{\frac{1}{\gamma}}}}$ where $\calR_T(\calF)$ denotes the (unnormalized) sequential Rademacher complexity. We note that these generic bounds depend on sequential/online complexity measures, and in this work we show that if the class $\calF$ is $L$-Lipschitz, then it is possible to achieve regret bounds that depend on statistical complexity measures with only a logarithmic dependence on $1/\gamma$ and $L$. 

\paragraph{Halfspaces and the Margin Assumption.} For the class of halfspaces $$\calH=\SET{x\mapsto \sign(\inangle{w, x}) \mid w \in \bbR^d},$$ the $\gamma$-margin assumption states that $\exists w^\star\in \bbR^d$ such that the online sequence $(x_1, y_1), \dots, (x_T, y_T) \in \bbR^d \times \SET{\pm 1}$ satisfies $y_t \inangle{w^\star, x_t}\geq \gamma, \forall 1\leq t \leq T$. Mistake bounds under the $\gamma$-margin assumption depend on the norm of the the data sequence, $\max_{1\leq t\leq T}\norm{}{x_t}$, and the corresponding dual norm of comparator halfspace, $\norm{\star}{w^\star}$. For example, the classic Perceptron algorithm \citep{rosenblatt1958perceptron} can learn halfspaces with a mistake bound of $\norm{2}{w^\star}^2B_2^2/{\gamma}^2$ \citep{novikoff1962convergence}, where $B_2=\max_{1\leq t \leq T}\norm{2}{x_t}$. And, the Winnow algorithm \citep{DBLP:journals/ml/Littlestone87} can learn halfspaces with a mistake bound of $O\inparen{\log d \cdot \norm{1}{w^\star}^2B_{\infty}^2/{\gamma}^2}$, where $B_\infty = \max_{1\leq t \leq T}\norm{\infty}{x_t}$. More generally, there is an algorithm due to \citet*{DBLP:journals/ml/GroveLS01} that can learn halfspaces with a mistake bound of ${(p-1)\norm{q}{w}^2B_p^2}/{\gamma^2}$, where $B_p=\max_{1\leq t \leq T}\norm{p}{x_t}$ and $2\leq p < \infty$. 

Under the same $\gamma$-margin assumption, it is also possible to achieve a different mistake bound of 

\noindent $O\inparen{d\log\inparen{{\norm{2}{w^\star}B_2}/{\gamma}}}$ \citep*{DBLP:conf/colt/Gilad-BachrachNT04}. See also \citep*{qian2024refined} for the extension of this bound to the agnostic case.
Note here that there is only a logarithmic dependence on the inverse margin, at the expense of a linear dependence on the dimenion $d$. \citet{DBLP:conf/colt/Gilad-BachrachNT04} showed that this can be achieved via cutting plane methods, but it is also possible to achieve this with the Halving algorithm via a covering argument as noted in \citep*{rakhlin2014statistical}. In \prettyref{sec:halfspaces}, we generalize these results to handle arbitrary norms and dual norms, beyond the $\ell_2$ norm. 

In terms of regret bounds, i.e., when the $\gamma$-margin assumption does not hold, the Perceptron algorithm discussed above will compete instead with the smallest achievable hinge loss $\hOPT$, with a regret bound of $\sqrt{T\cdot \norm{2}{w^\star}^2B_2^2/{\gamma}^2}$ \citep[see e.g., Corollary 1 in][]{DBLP:journals/corr/abs-1305-0208}. In \prettyref{sec:halfspaces}, we give online algorithms that compete with smallest achievable margin loss $\mOPT$, with a regret bound of $\sqrt{T\cdot d\log\inparen{{\norm{2}{w^\star}B_2}/{\gamma}}}$. 

\paragraph{Smoothed Online Learning.}
Another line of work that is closely related to our work is the study of smoothed online learning. In the smoothed online learning setting, the distribution of the data is assumed to be sampled from a distribution $\mathcal{D}_t$ with the property that the likliehood ratio $ \frac{d \mathcal{D}_t }{d \mu }  \leq 1 /\sigma $ where $\mu$ is a fixed measure referred to as the base measure and $\sigma$ is referred to the smoothness parameter. 
In its simplest form, this assumption allows one to prove regret bounds of the form $O\bigl(\sqrt{\vc(\calH) T \ln(T/\sigma)}\bigr)$\citep*{DBLP:journals/jacm/HaghtalabRS24} but several works have extended the range of results in this framework \citep*{DBLP:conf/colt/BlockDGR22, block2024oracle, block2024performance, DBLP:conf/nips/HaghtalabRS20, oracle-efficient, bhatt2023smoothed} 

\paragraph{Adversarially Robust Learning.} We note that a population version of the benchmark $\pOPT$ \eqref{eqn:OPT_input} has been studied before in agnostic adversarially robust PAC learning \citep*[see e.g.,][]{DBLP:conf/colt/MontasserHS19}, where the goal is to learn a predictor that is robustly correct on adversarial perturbations of test examples (e.g., within a $\gamma$-ball as in $\pOPT$), based on i.i.d.~training examples. We highlight that our result in \prettyref{thm:input-margin-upperbnd} implies as a corollary a new result for adversarially robust learning with tolerance, a relaxation of adversarially robust learning introduced by \citet*{DBLP:conf/alt/AshtianiPU23}. We defer the formal statement and proof to \prettyref{app:adversarial-learning}.

\section{Competing with an Optimal Predictor under Worst-Case Perturbations}

In this section, we consider minimizing regret relative to the smallest achievable error with class $\calH$ under worst-case perturbations of the online sequence $x_1,\dots, x_T$ of distance at most $\gamma$ away. 

\begin{thm}
\label{thm:input-margin-upperbnd}
    For any metric space $(\calX,\rho)$, any $\gamma > 0$, and any class $\calH\subseteq \calY^\calX$, \prettyref{alg:input-margin} guarantees for any sequence $(x_1,y_1),\dots, (x_T, y_T)$, an expected number of mistakes of
    \[\sum_{t=1}^{T} \Ex\ind[ \hat{y}_t \neq y_t ] - \pOPT \leq \sqrt{T\cdot \vc(\calH)\ln\inparen{\frac{e\abs{\Cover(\calX,\rho,\gamma)}}{\vc(\calH)}}}.\]
\end{thm}

\begin{algorithm2e}[H]
\caption{}
\label{alg:input-margin}
\SetKwInput{KwInput}{Input}                
\SetKwInput{KwOutput}{Output}              
\DontPrintSemicolon
  \KwInput{Domain $\calX$, Hypothesis Class $\calH$, parameter $\gamma > 0$.}
Let $\calZ$ be a $\gamma$-cover of $\calX$ where $\forall x\in \calX, \exists z\in \calZ$ such that $z \in B(x,\gamma)$.\;
Fix an arbitrary mapping $\phi: \calX \to \calZ$ such that for each $x\in \calX$, $\phi(x)\in B(x,\gamma)$.\;
Project the class $\calH$ onto the (finite) set $\calZ$ where we denote the resulting restriction by $\calH|_{\calZ} = \SET{ h|_{\calZ}: \calZ \to \calY \mid h\in \calH}$.\;
Initialize $P_1$ to be a uniform mixture over $\calH|_{\calZ}$, and set $\eta=\sqrt{8\log|\calH|_{\calZ}|/T}$. \;
\For{$1\leq t\leq T$}{
    Upon receiving $x_t \in \calX$ from the adversary, let $z_t = \phi(x_t) \in \calZ$.\;
    Draw a random predictor $h\sim P_t$ and predict $\hat{y}_t = h(z_t)$.\;
    Once the true label $y_t$ is revealed, we update all experts $h\in \calH|_\calZ$: \[P_{t+1}(h) = P_t(h) e^{-\eta\ind[h(z_t)\neq y_t]} / Z_t\] where $Z_t$ is a normalization constant.\;
}
\end{algorithm2e}
Before proceeding with the proof of \prettyref{thm:input-margin-upperbnd}, we sketch the main ideas below.

\paragraph{High-Level Strategy.} Recall the relaxed benchmark of $\pOPT$ \eqref{eqn:OPT_input} that we want to compete against. Given a hypothesis class $\calH$, the main conceptual step is the construction of a new notion of cover $\calC$ with respect to $\calH$ that satisfies the following property,
\[\forall h\in \calH, \exists c\in \calC, \forall (x,y) \in \calX \times \calY:~~\ind[c(x) \neq y] \leq \max_{z\in B(x, \gamma)} \ind\insquare{h(z)\neq y}.\]
With such a cover $\calC$ of $\calH$, it follows from the above property that for any sequence $(x_t,y_t)_{t=1}^{T}$: $\min_{c\in \calC} \sum_{t=1}^{T}\ind[c(x_t)\neq y_t] \leq \mOPT$. Thus, we can use any online learning algorithm for $\calC$ to compete with $\mOPT$. 

To this end, \prettyref{alg:input-margin} proceeds by constructing such a (finite) cover $\calC$ for $\calH$ as follows. First it construct a $\gamma$-cover $\calZ$ of the space $\calX$ with respect to the metric $\rho$. Then, it projects the class $\calH$ onto $\calZ$. This defines the set of experts to be used in the Multiplicative Weights algorithm. Observe that because on each round $t$, \prettyref{alg:input-margin} maps $x_t$ to a point $z_t$ in the cover $\calZ$ that is $\gamma$-close to $x_t$, if there is a predictor $h\in \calH$ such that $\forall{z\in B(x_t, \gamma)}, h(z)= y_t$, then the projection of $h$ onto $\calZ$ (which is among the experts being used) will predict $y_t$ for the point $z_t$. This is essentially how the cover $\calC$ for $\calH$ satisfies the property written above.

\begin{proof}[Proof of \prettyref{thm:input-margin-upperbnd}]
Recall from \prettyref{alg:input-margin} that $\calZ$ is a $\gamma$-cover of $\calX$ where $\forall x\in \calX, \exists z\in \calZ$ such that $z \in B(x,\gamma)$, and $\phi: \calX \to \calZ$ is a mapping such that for each $x\in \calX$, $\phi(x)\in B(x,\gamma)$. In the event that there are two $z,z'\in \calZ$ such that $z,z'\in B(x,\gamma)$, $\phi$ acts as a tie-breaker.

We start by showing that the projection $\calH|_{\calZ}$ (as defined in \prettyref{alg:input-margin}) satisfies the following ``covering'' property relative to $\calH$, 
 \begin{align}
     \forall h\in \calH, \exists h|_\calZ \in \calH|_{\calZ}, &\forall (x,y)\in \calX\times \calY:\\
     &\text{ if }\forall z\in B(x,\gamma), h(z) = y, \text{ then } h|_\calZ(\phi(x)) = y.
 \end{align}
To see this, fix an arbitrary $h\in \calH$ and let $h|_\calZ \in \calH|_{\calZ}$ be the projection/restriction of $h$ on $\calZ$. Observe that for any $(x,y)\in \calX\times \calY$, if $\forall z\in B(x,\gamma)$, $h(z) = y$, then it holds that $h|_\calZ(\phi(x)) = y$ because $\phi(x) \in \calZ \wedge \phi(x) \in B(x, \gamma)$, and $h$ and $h|_\calZ$ are equal on $\calZ$ by definition. From this ``covering'' property of $\calH|_\calZ$ relative to $\calH$, it immediately follows that for any $(x_1,y_1),\dots, (x_T,y_T)$,
\begin{equation}
\label{eqn:cover-opt}
    \min_{h\in \calH|_\calZ} \sum_{t=1}^{T} \ind\insquare{h(\phi(x_t))\neq y_t} \leq \min_{h\in \calH} \sum_{t=1}^{T} \max_{z'_t\in B(x_t, \gamma)} \ind\insquare{ h(z'_t) \neq y_t }.
\end{equation}
Finally, invoking the regret guarantee of Multiplicative Weights (\prettyref{lem:multiplicativeweights}) and combining it with \prettyref{eqn:cover-opt} tells us that
\begin{align*}
\label{eqn:reg}
        \sum_{t=1}^{T} \Ex_{h\sim P_t} \ind[h(\phi(x_t)) \neq y_t] &\leq \frac{\eta}{1-e^{-\eta}} \min_{h\in \calH|_{\calZ}} \sum_{t=1}^{T} \ind[h(\phi(x_t)) \neq y_t] + \frac{1}{1-e^{-\eta}} \ln \abs{\calH|_{\calZ}}\\
        &\leq \frac{\eta}{1-e^{-\eta}} \min_{h\in \calH} \sum_{t=1}^{T} \max_{z'_t\in B(x_t, \gamma)} \ind\insquare{ h(z'_t) \neq y_t } + \frac{1}{1-e^{-\eta}} \ln \abs{\calH|_{\calZ}}.
\end{align*}
By Sauer-Shelah-Perles Lemma, $\ln \abs{\calH|_{\calZ}}\leq {\vc(\calH)}\ln\inparen{\frac{e\abs{\calZ}}{\vc(\calH)}}$. Choosing a suitable step size $\eta$ concludes the proof.
\end{proof}

\begin{thm}
\label{thm:input-margin-lowerbnd}
For any metric space $(\calX,\rho)$, any $\gamma > 0$, and any $1 \leq d \leq |\Cover(\calX,\rho,2\gamma)|$, there exists a class $\calH \subseteq \calY^\calX$ with $\vc(\calH)=d$ such that for any (possibly randomized) online learner, there exists a sequence $(x_1,y_1),\dots, (x_T,y_T)$ where 
\[\sum_{t=1}^{T} \Ex\ind[ \hat{y}_t \neq y_t ] - \pOPT \geq \Omega\inparen{\sqrt{T\cdot \vc(\calH)\ln\inparen{{\frac{|\Cover(\calX,\rho,2\gamma)|}{\vc(\calH)}}}}}.\]
\end{thm}

\begin{rem}
\label{rem:dualVC}
We note that the class $\calH$ constructed in the lower bound in  \prettyref{thm:input-margin-lowerbnd} has constant dual VC dimension, $\vc^\star(\calH)\leq 1$. For example, this implies that it is not possible in general to replace dependence on the metric entropy of $\calX$ with dependence on the dual VC dimension of $\calH$.
\end{rem}

\begin{proof}
Let $\Pack(\calX,\rho,2\gamma)=\{x_1,\dots, x_N\}$ be a $2\gamma$-packing of $\calX$ with respect to metric $\rho$. By definition, we have the property that the $\gamma$-balls of points in the packing $\Pack(\calX,\rho,2\gamma)$ are disjoint, i.e., $\cap_{i=1}^{N} B_\gamma(x_i)=\emptyset$. It is also well known that $N=\abs{\Pack(\calX,\rho,2\gamma)} \geq |\Cover(\calX,\rho,2\gamma)|$. We now describe the construction of a class $\calH$ on the $\gamma$-balls of the packing, i.e., $\cup_{i=1}^{N} B_\gamma(x_i) \subseteq \calX$ (with a trivial $+1$ labeling on the remainder of $\calX$). 

We follow a construction due to \cite*[][Proof of Theorem 3.2]{DBLP:journals/jacm/HaghtalabRS24}. Specifically, without loss of generality, we fix the following ordering of the points in the packing: $x_1,\dots, x_N$. Let $1\leq d \leq N$. Divide the points into $d$ disjoint subsets $A_1, A_2, \dots, A_d$, where each $A_i$ contains at least $\floor{\nicefrac{N}{d}}$ points, i.e., $A_1=\SET{x_1,\dots, x_{\floor{\nicefrac{N}{d}}}}, A_2=\SET{x_{\floor{\nicefrac{N}{d}}+1}, \dots, x_{{2\floor{\nicefrac{N}{d}}}}}$, and so on. On each subset $A_i$, instantiate the class of thresholds, i.e. for each $\theta \in A_i$, let $h_\theta(z)=1$ for $z\in \cup_{x\in A_i, x\geq \theta} B_\gamma(x)$ and $0$ for $z \notin  \cup_{x\in A_i, x\geq \theta} B_\gamma(x)$. In words, $h_\theta$ labels the entire $\gamma$-balls of all $x < \theta$ with $0$ and labels the $\gamma$-balls of all $x\geq \theta$ with $1$ where $x\in A_i$, and  $h_\theta$ is zero everywhere else. For a $d$-tuple of thresholds, define
{
\setlength{\abovedisplayskip}{2pt}
\setlength{\belowdisplayskip}{2pt}
\[h_{\theta_1,\dots, \theta_d}(z) = \sum_{i=1}^{d} \ind\insquare{z \in \cup_{x\in A_i}B_\gamma(x)} h_{\theta_i}(z).\]
Then, the class $\calH$ is the set of all such functions
\[\calH = \SET{h_{\theta_1,\dots, \theta_d} \mid \theta_1 \in A_1, \dots, \theta_d \in A_d}.\]
}
Observe that $\vc(\calH)=d$. First, $\vc(\calH) \leq d$, because shattering $d+1$ points implies there is one subset $A_i$ where 2 points are shattered, but since in each $A_i$ the class $\calH$ behaves as a threshold, this is impossible. Second, $\vc(\calH) \geq d$, because any $d$ points $\tilde{x}_1 \in A_1, \tilde{x}_2 \in A_2, \dots, \tilde{x}_d \in A_d$ can be shattered by picking the thresholds $\theta_1 \in A_1,\dots, \theta_d \in A_d$ appropriately.

We next turn to describing the construction of a mistake tree of depth $d\log_2\inparen{\floor{\nicefrac{N}{d}}}$. Observe that for each $A_i$, we can construct a mistake tree of depth $\log_2\inparen{\floor{\nicefrac{N}{d}}}$ using instances in $A_i$. Because $\calH$ is defined as a disjoint union of $d$ thresholds, we can combine the mistake trees for $A_1,\dots, A_{d}$, by attaching a copy of the mistake tree for $A_{i+1}$ to each leaf of the mistake tree for $A_i$, recursively. This yields a mistake tree of depth $d\log_2\inparen{\floor{\nicefrac{N}{d}}}$. Given this mistake tree, it follows from standard arguments \citep*[see e.g., Lemma 14 in][]{DBLP:conf/colt/Ben-DavidPS09} that the regret is bounded from below by $\Omega(\sqrt{T\cdot d\log_2\inparen{\floor{\nicefrac{N}{d}}}})$. This concludes the proof.
\end{proof}

\section{Competing with a Gaussian-Smoothed Optimal Predictor}
In this section, we consider competing with a different relaxation, $\gOPT$ \eqref{eqn:OPT_gauss}, which is the smallest achievable error with class $\calH$ under random perturbations of $x_t$ drawn from a multivariate Gaussian distribution $\calN(0,\sigma^2)$.

\begin{thm}
\label{thm:gauss-upperbnd}
    For any $\calX\subseteq \bbR^d$, any $\sigma, \eps > 0$, for any class $\calH\subseteq \calY^{\bbR^d}$, \prettyref{alg:gaussian} guarantees for any sequence $(x_1,y_1),\dots, (x_T, y_T)$, an expected number of mistakes of
    \[\sum_{t=1}^{T} \Ex\ind[ \hat{y}_t \neq y_t ] - \gOPT \leq  \sqrt{T\cdot \vc(\calH)\ln\inparen{\frac{ce\abs{\Cover\inparen{\calX,\lVert \cdot \rVert_2, \sqrt{\pi/32}\cdot\sigma \eps}}}{\eps^2}}}.\]
\end{thm}

\begin{rem}
\label{rem:fat-shattering-gaussian}
The proof of \prettyref{thm:gauss-upperbnd} implies that the $\ell_\infty$ metric entropy relative to $m$ points of the real-valued class $\calF_\sigma = \SET{x\mapsto \Ex_{z\sim \calN}[h(x+\sigma z)]\mid h\in \calH}$ is $O\inparen{\vc(\calH)\log\inparen{\frac{m}{\eps^2}}}$.
\end{rem}

Before proceeding with the proof of \prettyref{thm:gauss-upperbnd}, we sketch the main ideas below.

\paragraph{High-Level Strategy.} Recall the relaxed benchmark of $\gOPT$ \eqref{eqn:OPT_gauss} that we want to compete against. Given a hypothesis class $\calH$, the main conceptual step is the construction of a new notion of cover $\calC$ with respect to $\calH$ that satisfies the following property,
\[\forall h\in \calH, \exists c\in \calC, \forall (x,y) \in \calX \times \calY:~~\ind[c(x) \neq y] \leq \ind[y\cdot\Ex_{z\sim \calN}[h(x+\sigma z)] \leq \eps].\]
With such a cover $\calC$ of $\calH$, it follows from the above property that for any sequence $(x_t,y_t)_{t=1}^{T}$: $\min_{c\in \calC} \sum_{t=1}^{T}\ind[c(x_t)\neq y_t] \leq \gOPT $. Thus, we can use any online learning algorithm for $\calC$ to compete with $\gOPT$. 

To this end, we construct in \prettyref{alg:gaussian} such a (finite) cover $\calC$ for $\calH$ as follows. First, we consider a covering $\calZ$ of the domain $\calX$ with respect to the $\ell_2$ metric. Second, for any $h\in \calH$ and any $y \in \calY$, the map $x \mapsto y\Ex_{z\sim \calN}[h(x+\sigma z)]$ is $O(1/\sigma)$-Lipchitz (\prettyref{lem:lipschitz}), therefore, foreach $x\in \calX$ there will be a point $\tilde{x}\in \calZ$ close enough to $x$ such that, $\ind[y\cdot\Ex_{z\sim \calN}[h(\tilde{x}+\sigma z)] \leq \eps/2]\leq \ind[y\cdot\Ex_{z\sim \calN}[h(x+\sigma z)] \leq \eps]$. Thus, it suffices to focus our attention on all possible behaviors of the class $\calH$ on the cover $\calZ$ with respect to the loss function $(\tilde{x},y)\mapsto \ind[y\cdot\Ex_{z\sim \calN}[h(\tilde{x}+\sigma z)] \leq \eps/2]$. Then, observe that this loss function is bounded from below by the empirical loss function $(\tilde{x}, y)\mapsto \ind\insquare{y\cdot \frac{1}{M}\sum_{i=1}^{M} h(\tilde{x} +\sigma z_i) \leq 0}$, where $z_1,\dots, z_M \sim \calN$ (\prettyref{lem:gauss-uniformconv}). Hence, projecting the class $\calH$ onto $\calZ$ and the Gaussian samples suffices to capture all possible behaviors of $\calH$, and gives us the cover $\calC$ that we need.

\begin{algorithm2e}[H]
\caption{}
\label{alg:gaussian}
\SetKwInput{KwInput}{Input}                
\SetKwInput{KwOutput}{Output}              
\DontPrintSemicolon
  \KwInput{Domain $\calX \subseteq \bbR^d$, Hypothesis Class $\calH$, $\sigma, \eps >0$.}
Set parameters $\tilde{\eps}=\eps/4$ and $\gamma = (\sigma\tilde{\eps})/\sqrt{2/\pi}$.\;
Let $\calZ$ be a $\gamma$-cover of $\calX$ where $\forall x\in \calX, \exists \tilde{x}\in \calZ$ such that $\lVert x-\tilde{x}\rVert_2\leq \gamma$.\;
Fix an arbitrary mapping $\phi: \calX \to \calZ$ such that for each $x\in \calX$, $\lVert x-\phi(x)\rVert_2\leq \gamma$.\;
Let $\calS$ be the set of all (noisy) points $\cup_{\tilde{x}\in \calZ}\{\tilde{x}+\sigma z_1, \dots, \tilde{x}+\sigma z_M\}$, where for each $\tilde{x}\in \calZ$ we invoke \prettyref{lem:gauss-uniformconv} to obtain $M=O\inparen{\frac{\vc(\calH)}{\tilde{\eps}^2}}$ points $z_1,\dots, z_M \in \bbR^d$ that form an $\tilde{\eps}$-approximation in the sense of \prettyref{eqn:uniformconv}.\;
Project the class $\calH$ onto the (finite) set $\calS$ where we denote the resulting restriction by $\calH|_{\calS} = \SET{ h|_{\calS}: \calS \to \calY \mid h\in \calH}$.\;
Initialize $P_1$ to be a uniform mixture over $\calH|_{\calS}$, and set $\eta = \sqrt{8\log|\calH|_{\calS}|/T}$. \;
\For{$1\leq t\leq T$}{
    Upon receiving $x_t \in \calX$ from the adversary, let $\tilde{x}_t = \phi(x_t) \in \calZ$.\;
    Draw a random predictor $h\sim P_t$ and predict $\hat{y}_t = +1$ if $\frac{1}{M}\sum_{i=1}^{M} \ind[h(\tilde{x}_t +\sigma z_i = +1)] \geq \nicefrac{1}{2}$, otherwise predict $\hat{y}_t = -1$.\;
    Once the true label $y_t$ is revealed, we update all experts $h\in \calH|_{\calS}$: 
    \[P_{t+1}(h) = \frac{P_t(h)}{ Z_t} \cdot \exp\inparen{-\eta\ind\insquare{\frac{1}{M}\sum_{i=1}^{M} \ind[h(\tilde{x}_t +\sigma z_i \neq y_t)] \geq \frac{1}{2} }} \]
    where $Z_t$ is a normalization constant.\;
}
\end{algorithm2e}

\begin{proof} [Proof of \prettyref{thm:gauss-upperbnd}] First, observe that by Steps 8 and 9 in \prettyref{alg:gaussian}, it holds that the (expected) number of mistakes made by \prettyref{alg:gaussian} on the sequence $(x_1, y_1), \dots, (x_T, y_T)$ satisfies
\begin{equation}
\label{eqn:mistakes}
\begin{split}
   \sum_{t=1}^{T} \Ex \ind\insquare{\hat{y}_t \neq y_t}  &= \sum_{t=1}^{T} \Ex_{h\sim P_t} \ind\insquare{\frac{1}{M}\sum_{i=1}^{M} \ind[h(\tilde{x}_t +\sigma z_i \neq y_t)] \geq \frac{1}{2}}.
\end{split}
\end{equation}
Next, we invoke the regret guarantee of Multiplicative Weights (\prettyref{lem:multiplicativeweights}) which tells us that
\begin{equation}
\label{eqn:reg-smooth}
\begin{split}
        \sum_{t=1}^{T} &\Ex_{h\sim P_t} \ind\insquare{\frac{1}{M}\sum_{i=1}^{M} \ind[h(\tilde{x}_t +\sigma z_i \neq y_t)] \geq \frac{1}{2}}\\
        &\leq \frac{\eta}{1-e^{-\eta}} \min_{h\in \calH|_{\calS}} \sum_{t=1}^{T} \ind\insquare{\frac{1}{M}\sum_{i=1}^{M} \ind[h(\tilde{x}_t +\sigma z_i \neq y_t)] \geq \frac{1}{2}}
        + \frac{1}{1-e^{-\eta}} \ln \abs{\calH|_{\calS}}.
\end{split}
\end{equation}
In the remainder of the proof, we first bound from above the benchmark above defined with the ``empirical'' loss function $(\tilde{x}, y)\mapsto \ind\insquare{y\cdot \frac{1}{M}\sum_{i=1}^{M} h(\tilde{x} +\sigma z_i) \leq 0}$ evaluated on $(\tilde{x}_t, y_t)_{t=1}^{T}$, by the benchmark defined with the ``population'' loss function $(\tilde{x},y)\mapsto \ind[y\cdot\Ex_{z\sim \calN}[h(\tilde{x}+\sigma z)] \leq 2\tilde{\eps}]$ evaluated on $(\tilde{x}_t, y_t)_{t=1}^{T}$,
\begin{equation}
    \label{eqn:bounding-opt}
    \min_{h\in \calH|_{\calS}} \sum_{t=1}^{T} \ind\insquare{\frac{1}{M}\sum_{i=1}^{M} \ind[h(\tilde{x}_t +\sigma z_i \neq y_t)] \geq \frac{1}{2}} \leq \min_{h\in \calH} \sum_{t=1}^{T} \ind\insquare{ \Prob_{z\sim \calN}\SET{h(\tilde{x}_t+ \sigma z )\neq y_t}  \geq \frac{1}{2} - \tilde{\eps}}.
\end{equation}
And, after that, we bound from above the benchmark with the ``population'' loss function $(\tilde{x},y)\mapsto \ind[y\cdot\Ex_{z\sim \calN}[h(\tilde{x}+\sigma z)] \leq 2\tilde{\eps}]$ evaluated on $(\tilde{x}_t, y_t)_{t=1}^{T}$, by $\gOPT$ (which is evaluated on the original sequence $(x_t, y_t)_{t=1}^{T}$ satisfying for all $1\leq t\leq T, \lVert x_t - \tilde{x}_t\rVert_2\leq \gamma$), 
\begin{equation}
    \label{eqn:bounding-opt2}
    \min_{h\in \calH} \sum_{t=1}^{T} \ind\insquare{ \Prob_{z\sim \calN}\SET{h(\tilde{x}_t+ \sigma z )\neq y_t}  \geq \frac{1}{2} - \tilde{\eps}} \leq \min_{h\in \calH} \sum_{t=1}^{T} \ind\insquare{ \Prob_{z\sim \calN}\SET{h(x_t+ \sigma z )\neq y_t}  \geq \frac{1}{2} - 2\tilde{\eps}}.
\end{equation}
We now proceed with proving Equations~\ref{eqn:bounding-opt} and \ref{eqn:bounding-opt2}. To prove \eqref{eqn:bounding-opt}, the next helper Lemma establishes that the ``empirical'' loss can be used to approximate the ``population'' loss in the following sense,
\begin{lem}
\label{lem:gauss-uniformconv}
    For any $\tilde{\eps}\in(0,1)$, any $\sigma >0$, and any $x\in \calX$, there exists $z_1,\dots, z_M \in \bbR^d$ where $M=O\inparen{\frac{\vc(\calH)}{\tilde{\eps}^2}}$such that 
    \begin{equation}
    \label{eqn:uniformconv}
    \forall y \in \SET{\pm 1}, \forall h \in \calH: \left| \Prob_{z\sim \calN}\SET{h(x+ \sigma z )= y} - \frac{1}{M} \sum_{i=1}^{M} \ind\insquare{h(x+ \sigma z_i )= y} \right| \leq \tilde{\eps},
\end{equation}
and as a result, 
\begin{equation}
    \label{eqn:boundempiricalbytrue}
    \forall y \in \SET{\pm 1}, \forall h \in \calH: \ind\insquare{\frac{1}{M}\sum_{i=1}^{M} \ind[h(x +\sigma z_i \neq y)] \geq \frac{1}{2}} \leq \ind\insquare{ \Prob_{z\sim \calN}\SET{h(x+ \sigma z )\neq y}  \geq \frac{1}{2} - \tilde{\eps}}.
\end{equation}
\end{lem}
\begin{proof}[Proof of \prettyref{lem:gauss-uniformconv}]
By invoking uniform convergence guarantees for the class $\calH$, we know that for any $\delta > 0$, with probability at least $1-\delta$ over the draw of $M=O\inparen{\frac{\vc(\calH) +\log(1/\delta)}{\tilde{\eps}^2}}$ i.i.d.~Gaussian samples $z_1,\dots, z_M\sim \calN$,
\[\forall y \in \SET{\pm 1}, \forall h \in \calH: \abs{ \Prob_{z\sim \calN}\SET{h(x+ \sigma z )= y} - \frac{1}{M} \sum_{i=1}^{M} \ind\insquare{h(x+ \sigma z_i )= y} } \leq \tilde{\eps}.
\]
Thus, \prettyref{eqn:uniformconv} follows by choosing $\delta=1/6$, for example. 

Next, to prove \prettyref{eqn:boundempiricalbytrue}, it suffices to show that when $\ind\insquare{ \Prob_{z\sim \calN}\SET{h(x+ \sigma z )\neq y}  \geq \frac{1}{2} - \tilde{\eps}} =0$, then $\ind\insquare{\frac{1}{M}\sum_{i=1}^{M} \ind[h(x +\sigma z_i \neq y)] \geq \frac{1}{2}} = 0$. To this end, suppose that $\Prob_{z\sim \calN}\SET{h(x+ \sigma z )\neq y}  < \frac{1}{2} - \tilde{\eps}$. By the uniform convergence guarantee (\prettyref{eqn:uniformconv}), we have
    \[\frac{1}{M}\sum_{i=1}^{M} \ind[h(x +\sigma z_i \neq y)] \leq \Prob_{z\sim \calN}\SET{h(x+ \sigma z )\neq y} +\tilde{\eps} < \frac{1}{2} -\tilde{\eps} +\tilde{\eps} = \frac{1}{2}.\]
    Thus, $\ind\insquare{\frac{1}{M}\sum_{i=1}^{M} \ind[h(x +\sigma z_i \neq y)] \geq \frac{1}{2}} = 0$. 
\end{proof}
To show that \prettyref{eqn:bounding-opt} holds, we invoke \prettyref{lem:gauss-uniformconv} on the sequence $(\tilde{x}_t, y_t)_{t=1}^{T}$, which implies that
\[\min_{h\in \calH} \sum_{t=1}^{T} \ind\insquare{\frac{1}{M}\sum_{i=1}^{M} \ind[h(\tilde{x}_t +\sigma z_i \neq y_t)] \geq \frac{1}{2}} \leq \min_{h\in \calH} \sum_{t=1}^{T} \ind\insquare{ \Prob_{z\sim \calN}\SET{h(\tilde{x}_t+ \sigma z )\neq y_t}  \geq \frac{1}{2} - \tilde{\eps}}.\]
Observe now that since $\calH|_{\calS}$ is the projection of $\calH$ onto the cover with noise $\calS$ (Step 5 in \prettyref{alg:gaussian}), every behavior induced by the class $\calH$ with respect to the ``empirical'' loss is witnessed by the projection $\calH|_{\calS}$,
\[\min_{h\in \calH|_{\calS}} \sum_{t=1}^{T} \ind\insquare{\frac{1}{M}\sum_{i=1}^{M} \ind[h(\tilde{x}_t +\sigma z_i \neq y_t)] \geq \frac{1}{2}}= \min_{h\in \calH} \sum_{t=1}^{T} \ind\insquare{\frac{1}{M}\sum_{i=1}^{M} \ind[h(\tilde{x}_t +\sigma z_i \neq y_t)] \geq \frac{1}{2}}.\]
Combining the above two equations implies \prettyref{eqn:bounding-opt}.

We now turn to proving \prettyref{eqn:bounding-opt2}. The next two helper Lemmas show that Gaussian smoothing induces Lipschitzness (\prettyref{lem:lipschitz}), and as a result we can relate the ``population'' loss on $(\tilde{x}_t, y_t)_{t=1}^{T}$ with the ``population'' loss on $(x_t, y_t)_{t=1}^{T}$ (\prettyref{lem:cover-to-original}).
\begin{lem}
\label{lem:lipschitz}
For any function $g:\bbR^d \to \SET{\pm1}$ and any $y\in  \SET{\pm1}$, the $\sigma$-smoothed map

$x \mapsto \Prob_{z\sim \calN}\SET{g(x+ \sigma z ) \neq y}$ is $L_\sigma$-Lipschitz where $L_\sigma = \frac{\sqrt{2/\pi}}{\sigma}$.
\end{lem}
\begin{proof}
    Invoke \prettyref{lem:smoothness} with the functions $x\mapsto \ind[g(x)\neq +1]$, $x\mapsto \ind[g(x)\neq -1]$.
\end{proof}
\begin{lem}
\label{lem:cover-to-original}
    For any $h\in \calH$, any $x,\tilde{x} \in \bbR^d$, any $y\in \SET{\pm 1}$, and any scalar $a\in(0,1)$
    \[\ind\insquare{ \Prob_{z\sim \calN}\SET{h(\tilde{x}+ \sigma z )\neq y}  \geq a + L_\sigma \lVert x - \tilde{x}\rVert_2} \leq \ind\insquare{ \Prob_{z\sim \calN}\SET{h(x+ \sigma z )\neq y}  \geq a}.\]
\end{lem}
\begin{proof}[Proof of \prettyref{lem:cover-to-original}]
    It suffices to show that when $\ind\insquare{ \Prob_{z\sim \calN}\SET{h(x+ \sigma z )\neq y}  \geq a} = 0$, then\\
    $\ind\insquare{ \Prob_{z\sim \calN}\SET{h(\tilde{x}+ \sigma z )\neq y}  \geq a + L_\sigma \lVert x - \tilde{x}\rVert_2} = 0$. To this end, suppose that $\Prob_{z\sim \calN}\SET{h(x+ \sigma z )\neq y}  < a$. Since the map $x \mapsto \Prob_{z\sim \calN}\SET{h(x+ \sigma z ) \neq y}$ is $L_\sigma$-Lipschitz, it holds that
    \[\Prob_{z\sim \calN}\SET{h(\tilde{x}+ \sigma z )\neq y} \leq \Prob_{z\sim \calN}\SET{h(x+ \sigma z ) \neq y} + L_\sigma \lVert x - \tilde{x}\rVert_2 < a + L_\sigma \lVert x - \tilde{x}\rVert_2.\]
    Thus, $\ind\insquare{ \Prob_{z\sim \calN}\SET{h(\tilde{x}+ \sigma z )\neq y}  \geq a + L_\sigma \lVert x - \tilde{x}\rVert_2} = 0$.
\end{proof}
\noindent \prettyref{eqn:bounding-opt2} follows by invoking \prettyref{lem:cover-to-original} with scalar $a=\nicefrac{1}{2}-2\tilde{\eps}$ and noting that for all $1\leq t\leq T, \lVert x_t - \tilde{x}_t\rVert_2\leq \gamma$ where $\gamma = \frac{\tilde{\eps}}{L_\sigma}$.

To conclude the proof of \prettyref{thm:gauss-upperbnd}, putting things together, Equations~\ref{eqn:mistakes}, \ref{eqn:reg-smooth}, \ref{eqn:bounding-opt}, \ref{eqn:bounding-opt2} and the choice of $\tilde{\eps}=\eps/4$ imply that the regret of the online learner is bounded from above by
\begin{equation}
    \label{eqn:reg-smooth-tog}
    \sum_{t=1}^{T} \Ex \ind\insquare{\hat{y}_t \neq y_t} \leq \frac{\eta}{1-e^{-\eta}} \min_{h\in \calH} \sum_{t=1}^{T} \ind\insquare{ \Prob_{z\sim \calN}\SET{h(x_t+ \sigma z )\neq y_t}  \geq \frac{1}{2} - \frac{\eps}{2}} + \frac{1}{1-e^{-\eta}} \ln \abs{\calH|_{\calS}}.
\end{equation}
By Sauer-Shelah-Perles Lemma, $\abs{\calH|_{\calS}}\leq \inparen{\frac{e\abs{\calZ}M}{\vc(\calH)}}^{\vc(\calH)}\leq \inparen{c\frac{e |\calZ|}{\eps^2}}^{\vc(\calH)}$. Choosing a suitable step size $\eta$ concludes the proof.
\end{proof}

\begin{thm}
\label{thm:gauss-lowerbnd}
    For $\calX=[0,1]$, any $\sigma >0$ and $0 < \eps < \nicefrac{1}{2}$, any $1\leq d \leq \abs{\Cover(\calX,\abs{\cdot}, 4\sigma\eps)}$, there exists a class $\calH \subseteq \calY^\calX$ with $\vc(\calH)=d$ such that for any (possibly randomized) online learner, there exists a sequence $(x_1,y_1),\dots, (x_T,y_T)$ where 
    \[
        \sum_{t=1}^{T} \Ex\ind[ \hat{y}_t \neq y_t ] - \gOPT \geq \Omega\inparen{\sqrt{T\cdot \vc(\calH)\ln\inparen{{\frac{\abs{\Cover(\calX,\abs{\cdot}, 4\sigma\eps)} }{\vc(\calH)}}}}}.
    \]
\end{thm}
At a high-level, we use a similar construction to \prettyref{thm:input-margin-lowerbnd} of a disjoint union of thresholds that are carefully spaced on the interval $[0,1]$.

\begin{proof}[Proof of \prettyref{thm:gauss-lowerbnd}]
    Let $\Phi(r)=\Prob_{z\sim \calN}\insquare{z \leq r}$ denote the CDF of a standard Gaussian. Set $\alpha =\sigma\Phi^{-1}(\nicefrac{1}{2} +\nicefrac{\eps}{2})$. By choice of $\alpha$, we have $\Prob_{z\sim \calN}\insquare{\sigma z \leq \alpha}=\nicefrac{1}{2} +\nicefrac{\eps}{2}$, and symmetrically, $\Prob_{z\sim \calN}\insquare{\sigma z \geq -\alpha}=\nicefrac{1}{2} +\nicefrac{\eps}{2}$. For any $\theta \in [0,1]$, define a threshold function $h_{\theta}: [0,1] \to \SET{\pm 1}$ such that $h_\theta(x) = -1$ if $x \leq \theta$ and $h_\theta(x) = +1$ if $x > \theta$. Observe that for any $x$ such that $x +\alpha < \theta$ it holds that $\Prob_{z\sim \calN}\insquare{x + \sigma z < \theta} = \Prob_{z\sim \calN}\insquare{\sigma z < \theta - x} \geq \Prob_{z\sim \calN}\insquare{\sigma z \leq \alpha} = \nicefrac{1}{2} +\nicefrac{\eps}{2}$. Symmetrically, for any $x$ such that $x - \alpha > \theta$ it holds that $\Prob_{z\sim \calN}\insquare{x + \sigma z > \theta} = \Prob_{z\sim \calN}\insquare{\sigma z > \theta - x} \geq \Prob_{z\sim \calN}\insquare{\sigma z \leq -\alpha} = \nicefrac{1}{2} +\nicefrac{\eps}{2}$. Thus, for any $x$ such that $\abs{x-\theta} > \alpha$ it holds that $h_\theta(x)\cdot \Ex_{z\sim \calN}[h_\theta(x+\sigma z)] \geq \eps$. That is, any $x$ that is distance greater than $\alpha$ away from $\theta$ will satisfy a margin at least $\eps$ under random Gaussian noise. 

    Let $\Pack([0,1],\abs{\cdot},2\alpha)=\{\theta_1,\dots, \theta_N\}$ be a $2\alpha$-packing of $[0,1]$ with respect to metric $\abs{\cdot}$, that is, $\min_{i\neq j} \abs{\theta_i - \theta_j} > 2\alpha$. We follow a construction due to \citet*[][Proof of Theorem 3.2]{DBLP:journals/jacm/HaghtalabRS24}. Specifically, without loss of generality, we fix the following ordering of the thresholds in the packing: $\theta_1,\dots, 
    \theta_N$. Let $1\leq d \leq N$. Divide the thresholds into $d$ disjoint subsets $A_1, A_2, \dots, A_d$, where each $A_i$ contains at least $\floor{\nicefrac{N}{d}}$ thresholds, i.e., $A_1=\SET{\theta_1,\dots, \theta_{\floor{\nicefrac{N}{d}}}}, A_2=\SET{\theta_{\floor{\nicefrac{N}{d}}+1}, \dots, \theta_{{2\floor{\nicefrac{N}{d}}}}}$, and so on. For a $d$-tuple of thresholds, define
\[h_{\theta_1,\dots, \theta_d}(x) = \sum_{i=1}^{d} \ind\insquare{ \theta_{(i-1)\floor{\nicefrac{N}{d}}+1} \leq x \leq  \theta_{i\floor{\nicefrac{N}{d}}}} h_{\theta_i}(x).\]
Then, the class $\calH$ is the set of all such functions
\[\calH = \SET{h_{\theta_1,\dots, \theta_d} \mid \theta_1 \in A_1, \dots, \theta_d \in A_d}.\]
Observe that $\vc(\calH)=d$. First, $\vc(\calH) \leq d$, because shattering $d+1$ points implies there is one subset $A_i$ where 2 points are shattered, but since in each $A_i$ the class $\calH$ behaves as a threshold, this is impossible. Second, $\vc(\calH) \geq d$, because any $d$ points $\tilde{x}_1 \in A_1, \tilde{x}_2 \in A_2, \dots, \tilde{x}_d \in A_d$ can be shattered by picking the thresholds $\theta_1 \in A_1,\dots, \theta_d \in A_d$ appropriately. Observe also that $h_\theta(x) = \sign(\Ex_{z\sim \calN}h_\theta(x+\sigma z))$ because if $x\leq \theta$ then $\Prob_{z\sim \calN}\insquare{x+\sigma z \leq \theta}\leq 1/2$ and if $x > \theta$ then $\Prob_{z\sim \calN}\insquare{x+\sigma z > \theta} > 1/2$. Thus, the class $\calH$ is closed under the $\sigma$-Gaussian smoothing operation. Therefore, the VC dimension of this class after Gaussian smoothing remains $d$. 

We next turn to describing the construction of a mistake tree of depth $d\log_2\inparen{\floor{\nicefrac{N}{d}}}$. For each subset $A_i =\SET{\theta_{(i-1)\floor{\nicefrac{N}{d}}+1}, \dots, \theta_{i\floor{\nicefrac{N}{d}}}}$, we pick instances $x\in \calX$ that are exactly halfway between consecutive thresholds in $A_i$, let $B_i= \SET{x_{i_1}, \dots, x_{i_{\floor{\nicefrac{N}{d}}-1}}}$ denote such instances. Given that $\min_{i\neq j} \abs{\theta_i - \theta_j} > 2\alpha$, by the choice of $\alpha$, we can construct a Littlestone tree using the instances in $B_i$ of depth $\log_{2}\inparen{\floor{\nicefrac{N}{d}}}$ such that each path is realized by a threshold $\theta \in A_i$ that satisfies $\ind[h_\theta(x)\cdot\Ex_{z\sim \calN}[h_\theta(x+\sigma z)] \leq \eps]=0$ for the $x$ instances along this path. Finally, because $\calH$ is defined as a disjoint union of $d$ thresholds, we can combine the mistake trees for $A_1,\dots, A_{d}$, by attaching a copy of the mistake tree for $A_{i+1}$ to each leaf of the mistake tree for $A_i$, recursively. This yields a mistake tree of depth $d\log_2\inparen{\floor{\nicefrac{N}{d}}}$. Given this mistake tree, it follows from standard arguments \citep*[see e.g., Lemma 14 in][]{DBLP:conf/colt/Ben-DavidPS09} that the regret is bounded from below by $\Omega(\sqrt{T\cdot d\log_2\inparen{\floor{\nicefrac{N}{d}}}})$. To conclude the proof, we note that for $0 < \eps < 1/2$, $\Phi^{-1}(\nicefrac{1}{2} +\nicefrac{\eps}{2}) \leq 2\epsilon$, implying that $N=\abs{\Pack(\calX,\abs{\cdot},2\sigma \Phi^{-1}(\nicefrac{1}{2} +\nicefrac{\eps}{2}))} \geq \abs{\Pack(\calX,\abs{\cdot}, 4\sigma\eps)}$. Finally, the packing number is bounded from below by the covering number, $\abs{\Pack(\calX,\abs{\cdot}, 4\sigma\eps)} \geq \abs{\Cover(\calX,\abs{\cdot}, 4\sigma\eps)}$.
\end{proof}

\section{Competing with an Optimal Predictor with a Margin}
\label{sec:classical-margin}

In this section, we revisit the classical notion of margin for a (real-valued) class $\calF \subseteq [-1, +1]^\calX$. Specifically, we consider competing with the smallest achievable error with functions $f\in \calF$ that have a margin of $\gamma$.

\begin{thm}
\label{thm:output-margin-upperbnd}
    For any metric space $(\calX,\rho)$ and any $\gamma > 0$, for any function class $\calF\subseteq [-1,1]^\calX$ that is $L$-Lipschitz relative to $\rho$, there exists an online learner such that for any sequence $(x_1,y_1),\dots, (x_T, y_T)$, the expected number of mistakes satisfies
    \[\sum_{t=1}^{T} \Ex\ind[ \hat{y}_t \neq y_t ] - \mOPT \lesssim \sqrt{T\cdot \min\SET{G_0, G_{\gamma/4}}},\]
    where $G_0$ is the size of the projection of the binary class $\sign(\calF)$ onto the cover $\Cover(\calX,\rho,\nicefrac{\gamma}{2L})$ which satisfies $G_0 \leq {\vc(\calF)}\ln\inparen{\frac{e\abs{\Cover(\calX,\rho,\nicefrac{\gamma}{2L})}}{\vc(\calF)}}$ , and $G_{\gamma/4}$ is the size of the smallest $\ell_\infty$ $\gamma/4$-cover of $\calF$ on $\Cover(\calX,\rho,\nicefrac{\gamma}{2L})$ which satisfies for any $\alpha >0$, there are constants $c_1, c_2, c_3 > 0$ such that $G_{\gamma/4} \leq c_1 \fat_\calF\inparen{c_2\alpha \frac{\gamma}{4}} \ln^{1+\alpha}\inparen{\frac{c_3 \abs{\Cover(\calX,\rho,\nicefrac{\gamma}{2L})}}{\fat_\calF\inparen{c_2\frac{\gamma}{4}}\cdot\frac{\gamma}{4}}}$.
\end{thm}
Before proceeding with the proof of \prettyref{thm:output-margin-upperbnd}, we sketch the main ideas below.
\paragraph{High-Level Strategy.} Recall the relaxed benchmark of $\mOPT$ \eqref{eqn:OPT_output} that we want to compete against. Given a function class $\calF$, the main conceptual step is the construction of a new notion of cover $\calG$ with respect to $\calF$ that satisfies the following property,
\[\forall f\in \calF, \exists g\in \calG, \forall (x,y) \in \calX \times \calY:~~\ind[g(x) \neq y] \leq \ind[y\cdot f(x) \leq \gamma].\]
With such a cover $\calG$ of $\calF$, it follows from the above property that for any sequence $(x_t,y_t)^{T}_{t=1}$: $\min_{g\in \calG} \sum_{t=1}^{T}\ind[g(x_t)\neq y_t] \leq \mOPT $. Thus, we can use any online learning algorithm for $\calG$ to compete with $\mOPT$. 

In the proof, we construct such a cover $\calG$ in two ways. We first take a covering covering $\calZ$ of the domain $\calX$ with respect to metric $\rho$ at scale $\gamma/2L$. Then, since the class $\calF$ is $L$-Lipschitz, it follows that for any $f\in\calF$ and any $(x,y)\in\calX\times \calY$, choosing $z\in\calZ$ such that $\rho(x,z)\leq \gamma/2L$, we have $\ind[yf(z)\leq \gamma/2] \leq \ind[yf(x)\leq \gamma]$. Thus to construct $\calG$, we can either use the projection of the binary class $\sign(\calF)$ on $\calZ$, or use a cover of $\calF$ with respect to $\ell_\infty$ metric on $\calZ$ at scale $\gamma/4$. 

\begin{proof}[Proof of \prettyref{thm:output-margin-upperbnd}]
    Let $\calZ=\Cover(\calX, \rho, \gamma/2L)$ denote a cover of $\calX$ relative to metric $\rho$ at scale $\gamma/2L$. 
    
    \noindent\textit{A Zero-Scale Cover of $\calF$.} Let $\calG_0=\{g: \calZ \to \calY\}$ denote the projection of the binary class $\sign(\calF)$ onto $\calZ$. By definition, $\calG_0$ satisfies the following property
    \[\forall f\in \calF, \exists g\in \calG_0, \forall z\in \calZ: \sign(f)(z)=g(z).\]
    Given that the class $\calF$ is $L$-Lipschitz relative to $\rho$, it follows that the cover $\calG_0$ satisfies the following property
    \begin{equation}
        \label{eqn:zero-cover}
        \forall f\in \calF, \exists g\in \calG_0, \forall (x,y)\in \calX\times \calY, \forall z\in \calZ \text{ s.t. }\rho(x,z)\leq \frac{\gamma}{2L}, \ind[ g(z) \neq y] \leq \ind[yf(x)\leq \gamma],
    \end{equation}
    because whenever $yf(x)> \gamma$, for any $z\in \calZ$ such that $\rho(x, z)\leq \gamma/2L$, it holds that $yf(z)\geq yf(x) - \gamma/2 > \gamma - \gamma/2 = \gamma/2$ and therefore, $g(z) = \sign(f(z)) = y$. \prettyref{eqn:zero-cover} then implies that for any sequence $(x_1, y_1),\dots, (x_T, y_T) \in \calX\times \calY$, and any $z_1,\dots, z_T\in \calZ$ such that $\forall 1\leq t \leq T, \rho(x_t, z_t) \leq \gamma/2L $, 
    \begin{equation}
        \label{eqn:best-in-zero-cover}
        \min_{g\in \calG_0} \sum_{t=1}^{T} \ind\insquare{g(z_t)\neq y_t} \leq \min_{f\in\calF} \sum_{t=1}^{T} \ind\insquare{y_tf(x_t)\leq \gamma}. 
    \end{equation}
     By Sauer-Shelah-Perles Lemma, $\log\abs{\calG_0}\leq {\vc(\calF)}\log\inparen{\frac{e\abs{\calZ}}{\vc(\calF)}}$.
    
    \noindent\textit{A Scale-Sensitive Cover of $\calF$.} Let $\calG_{\gamma/4}=\SET{g: \calZ\to [-1, +1]}$ be a cover for $\calF$ relative to the $\ell_\infty$ metric on $\calZ$ at scale $\gamma/4$. In other words, $\calG_{\gamma/4}$ is a cover that satisfies the following property
    \[\forall f\in \calF, \exists g\in \calG_{\gamma/4}, \sup_{z\in \calZ} \abs{f(z)-g(z)}\leq \frac{\gamma}{4}.\]
    Given that the class $\calF$ is $L$-Lipschitz relative to $\rho$, it follows that the cover $\calG_{\gamma/4}$ satisfies the following property
    \begin{equation}
        \label{eqn:scale-cover}
        \forall f\in \calF, \exists g\in \calG_{\gamma/4}, \forall (x,y)\in \calX\times \calY, \forall z\in \calZ \text{ s.t. }\rho(x,z)\leq \frac{\gamma}{2L}, \ind[ yg(z) \leq \gamma/4] \leq \ind[yf(x)\leq \gamma],
    \end{equation}
    because whenever $yf(x)> \gamma$, for any $z\in \calZ$ such that $\rho(x, z)\leq \gamma/2L$, it holds that $yg(z) \geq yf(z) - \gamma/4 \geq yf(x) - \gamma/4 - L \cdot \gamma/2L > \gamma - \gamma/4 - \gamma/2 = \gamma/ 4$. \prettyref{eqn:scale-cover} then implies that for any sequence $(x_1, y_1),\dots, (x_T, y_T) \in \calX\times \calY$, and any $z_1,\dots, z_T\in \calZ$ such that $\forall 1\leq t \leq T, \rho(x_t, z_t) \leq \gamma/2L $, 
    \begin{equation}
        \label{eqn:best-in-scale-cover}
        \min_{g\in \calG_{\gamma/4}} \sum_{t=1}^{T} \ind\insquare{y_tg(z_t)\leq \gamma/4} \leq \min_{f\in\calF} \sum_{t=1}^{T} \ind\insquare{y_tf(x_t)\leq \gamma}.
    \end{equation}
    We can bound the size of the scale-sensitive cover $\calG_{\gamma/4}$ in terms of the fat-shattering dimension of $\calF$. For example, see \citet*{rudelson2006combinatorics}[Theorem 4.4] or \citet*{DBLP:conf/colt/BlockDGR22}[Theorem 23], for any $\alpha >0$, there are constants $c_1, c_2, c_3 > 0$ such that
    \[\log \abs{\calG_{\gamma/4}} \leq c_1 \fat_\calF\inparen{c_2\alpha \frac{\gamma}{4}} \log^{1+\alpha}\inparen{\frac{c_3 \abs{\calZ}}{\fat_\calF\inparen{c_2\alpha\frac{\gamma}{4}}\cdot\frac{\gamma}{4}}}.\]
    \noindent\textit{Regret Guarantees.} We use Multiplicative Weights with the set of experts $\calC = \calG_0$ or $\calC= \calG_{\gamma/4}$. Invoking the regret guarantee of Multiplicative Weights (\prettyref{lem:multiplicativeweights}) implies that
    \begin{equation}
        \sum_{t=1}^{T} \Ex\ind[\hat{y}_t \neq y_t] \leq \frac{\eta}{1-e^{-\eta}} \min_{c\in \calC} \sum_{t=1}^{T} \ind[y_tc(z_t)\leq 0] + \frac{1}{1-e^{-\eta}} \ln \abs{\calC}.
    \end{equation}
    By \prettyref{eqn:best-in-zero-cover} and \prettyref{eqn:best-in-scale-cover}, it then follows that
    \begin{equation}
        \sum_{t=1}^{T} \Ex\ind[\hat{y}_t \neq y_t] \leq \frac{\eta}{1-e^{-\eta}} \min_{f\in\calF} \sum_{t=1}^{T} \ind\insquare{y_tf(x_t)\leq \gamma} + \frac{1}{1-e^{-\eta}} \ln \abs{\calC}.
    \end{equation}
    Choosing a suitable $\eta > 0$ concludes the proof. 
\end{proof}

\begin{thm}
\label{thm:output-margin-lowerbnd}
    For $\calX=[0,1]$, any $0<\gamma <1/2$, any Lipschitz constant $L\in [0,\infty)$, any $1\leq d \leq \abs{\Cover(\calX,\abs{\cdot},  \sqrt{32/\pi} \cdot \frac{\gamma}{L})}$, there exists a class $\calF\subseteq [-1,1]^\calX$ that is $L$-Lipschitz relative to $\abs{\cdot}$ with $\vc(\calF)=d$ such that for any (possibly randomized) online learner, there exists a sequence $(x_1,y_1),\dots, (x_T,y_T)$ where
    \[\sum_{t=1}^{T} \Ex\ind[ \hat{y}_t \neq y_t ] - \mOPT \geq \Omega\inparen{\sqrt{T\cdot \vc(\calF)\ln\inparen{{\frac{\abs{\Cover\inparen{\calX,\abs{\cdot}, \sqrt{32/\pi} \cdot \frac{\gamma}{L}}} }{\vc(\calF)}}}}}.\]
\end{thm}

\begin{proof}
The proof follows directly from \prettyref{thm:gauss-lowerbnd} and the construction used in its proof. In particular, we use $\sigma = \sqrt{2/\pi}/L$, $\eps = \gamma$, and $\calF=\SET{ x\mapsto \Ex_{z\sim \calN}h(x+\sigma z)\mid h \in \calH}$. Note that $\calF$ is $L$-Lipschitz by \prettyref{lem:lipschitz}, and $\vc(\calF) = \vc(\calH)$ (see proof of \prettyref{thm:gauss-lowerbnd}).
\end{proof}

\section{Refined Results for Halfspaces}
\label{sec:halfspaces}

In this section, we focus specifically on the class of homogeneous halfspaces,
$$\mathcal{H}=\SET{x \mapsto \sign\inparen{\inangle{w,x}}: w \in \mathbb{R}^d}.$$

As mentioned in \prettyref{clm:halfspaces-equivalence}, there is an equivalence for halfspaces between the three relaxed benchmarks that we study in this paper $\pOPT$ \eqref{eqn:OPT_input}, $\gOPT$ \eqref{eqn:OPT_gauss}, and $\mOPT$ \eqref{eqn:OPT_output}. So, we focus on  $\mOPT$ \eqref{eqn:OPT_output}. Below, we generalize results from the literature which considered the $\ell_2$-norm and the realizable case \citep*{DBLP:conf/colt/Gilad-BachrachNT04, rakhlin2014statistical}, to handle arbitrary norms and the agnostic case.

\begin{thm}
\label{thm:halfspaces-upperbound}
    For any normed vector space $(\calX,\norm{}{\cdot})$ where $\mathcal{X}\subseteq \mathbb{R}^d$ and $B=\sup_{x\in\calX} \|x\| < \infty$, and any $\gamma > 0$, there is an online learner such that for any sequence $(x_1,y_1),\dots, (x_T, y_T)$, the  expected number of mistakes satisfies
    \[\sum_{t=1}^{T} \Ex\ind[ \hat{y}_t \neq y_t ] - \min_{w\in \bbR^d, \norm{\star}{w}=1} \sum_{t=1}^{T} \ind\insquare{y_t \inangle{w, x_t} \leq \gamma} \leq \sqrt{T\cdot  d\ln\inparen{1+\frac{2 B}{\gamma}}}.\]
\end{thm}

\begin{rem}
    We can also compete with $\min_{w\in \bbR^d} \sum_{t=1}^{T} \ind[y_t \inangle{w, x_t} \leq \gamma]$, i.e., without restricting to unit-norm $w$'s, when we know in advance the norm $\norm{\star}{w}$ of the competitor or an upper bound on it. 
\end{rem}

Before proceeding with the proof of \prettyref{thm:halfspaces-upperbound}, we highlight that we conceptually follow the same strategy as in our earlier generic results of constructing a suitable (finite) cover $\calC$ for $\calH$. However, to construct $\calC$, we utilize the parametric structure of halfspaces and directly cover the space of parameters $W=\SET{w\in \mathbb{R}^d: \norm{\star}{w}=1}$, instead of first covering $\calX$ and then projecting $\calH$ onto this cover. This enables us to obtain a $O\inparen{\sqrt{Td\log(1/\gamma)}}$ regret bound, as opposed to a $O\inparen{\sqrt{Td^2\log(1/\gamma)}}$ regret bound implied by our earlier generic results.

\begin{proof}[Proof of \prettyref{thm:halfspaces-upperbound}]
We start with describing the construction of the online learner. Let $\beta=\gamma /B$ and $C_\beta$ be a minimal size $\beta$-cover of the unit-ball $W=\SET{w\in \mathbb{R}^d: \norm{\star}{w}=1}$ with respect to the dual norm $\norm{\star}{\cdot}$, where $\forall w\in W, \exists \tilde{w}\in C_\beta$ such that $\norm{\star}{w-\tilde{w}} \leq \beta$. Let 
\[H_\beta = \SET{h_w: x \mapsto \sign(\inangle{w, x}) |~w \in C_\beta}\]
be the set of halfspaces induced by the cover $C_\beta$. 

Next, we run the Multiplicative Weights algorithm with the halfspaces in $H_\beta$ as experts. Specifically, we start with a uniform mixture $P_1$ over $H_\beta$ and some learning rate $\eta > 0$. Then, on rounds $t=1,\dots, T$:
    \begin{enumerate}
        \item Upon receiving $x_t \in \calX$ from adversary, draw a random predictor $h_w\sim P_t$ and predict $\hat{y}_t = h_w(x_t)$.
        \item Once the true label $y_t$ is revealed, we update for all $h_w\in H_\beta$: $P_{t+1}(h_w) = P_t(h_w) e^{-\eta\ind[h_w(x_t)\neq y_t]} / Z_t$, where $Z_t$ is a normalization constant.
    \end{enumerate}
    \paragraph{Analysis.}
    We first invoke the regret guarantee of Multiplicative Weights (\prettyref{lem:multiplicativeweights}) which tells us that
    \begin{equation}
    \label{eqn:reg-bnd}
        \sum_{t=1}^{T} \Ex_{h_w\sim P_t} \ind[h_w(x_t) \neq y_t] \leq \frac{\eta}{1-e^{-\eta}} \min_{h_w\in H_\beta} \sum_{t=1}^{T} \ind[h_w(x_t) \neq y_t] + \frac{1}{1-e^{-\eta}} \ln \abs{H_\beta}.
    \end{equation}
We now argue that the cover $C_\beta$ for $W$ satisfies the following property,
\begin{equation}
    \label{eqn:covering-w}
    \forall w\in W, \exists \tilde{w}\in C_\beta, \forall (x,y) \in \calX \times \calY:~~\ind[y\inangle{\tilde{w}, x} \leq 0] \leq \ind\insquare{y \inangle{w, x} \leq \beta \norm{}{x}}.
\end{equation}
To see this, observe that for any $w\in W$ and $\tilde{w}\in C_\beta$ such that 
$\norm{\star}{w-\tilde{w}} \leq \beta$, and any $(x,y) \in \calX \times \calY$,
\[y\inangle{\tilde{w}, x} = y\inangle{w + (\tilde{w}-w), x} = y\inangle{w, x} + y\inangle{\tilde{w}-w, x} \geq y\inangle{w,x} - \beta \norm{}{x},\]
where the last inequality follows from the fact that $\abs{\inangle{\tilde{w}-w, x}} \leq \norm{\star}{\tilde{w}-w} \norm{}{x} \leq \beta \norm{}{x}$. Thus, it follows that if $y\inangle{w,x} > \beta \norm{}{x}$ then $y\inangle{\tilde{w}, x}>0$. 

For any sequence $(x_t,y_t)_{t=1}^{T}$, it follows from \prettyref{eqn:covering-w} and the choice of $\beta=\gamma /B$ that 
\begin{equation}
    \label{eqn:compare-opts}
    \min_{h_w\in H_\beta} \sum_{t=1}^{T} \ind[h_w(x_t) \neq y_t] = \min_{w \in C_\beta} \sum_{t=1}^{T} \ind\insquare{y_t \inangle{ w, x_t}\leq 0} \leq \min_{w'\in W} \sum_{t=1}^T \ind\insquare{y_t \inangle{w', x_t} \leq \gamma}. 
\end{equation}
Combining \prettyref{eqn:reg-bnd} and \prettyref{eqn:compare-opts}, we have
\begin{equation}
\sum_{t=1}^{T} \Ex_{h_w\sim P_t} \ind[h_w(x_t) \neq y_t] \leq \frac{\eta}{1-e^{-\eta}} \min_{w'\in W} \sum_{t=1}^T \ind\insquare{y_t \inangle{w', x_t} \leq \gamma} + \frac{\ln \abs{C_\beta}}{1-e^{-\eta}}.
\end{equation}
Choosing a suitable step size $\eta$, and noting that $\ln |C_\beta| \leq d\ln\inparen{1+\nicefrac{2}{\beta}}$ \citep[e.g., Corollary 27.4 in][]{Polyanskiy_Wu_2025} concludes the proof.
\end{proof}

We now proceed with proving the equivalence for halfspaces between the three relaxed benchmarks that we study in this paper $\pOPT$ \eqref{eqn:OPT_input}, $\gOPT$ \eqref{eqn:OPT_gauss}, and $\mOPT$ \eqref{eqn:OPT_output}. We start with stating two helper Lemmas that relate the margin-loss for halfspaces (used in defining $\mOPT$)  with the corresponding losses used in defining $\pOPT$ and $\gOPT$.
\begin{lem} [Lemma 4.2 in \citep*{DBLP:conf/icml/MontasserGDS20}]
\label{lem:perturbation-halfspace}
For any $w, x \in \mathbb{R}^d$ and any $y \in \SET{\pm 1}$,
\[
\max_{z\in B(x, \gamma)} \ind[y \inangle{w,z}\leq 0] = \ind\insquare{\frac{y\inangle{{w},x}}{\norm{\star}{w}}\leq \gamma}.
\]
\end{lem}

\begin{lem} 
\label{lem:gaussian-margin-rel}
For any $w\in \bbR^d$, any $(x,y) \in \bbR^d\times \calY$, and any $\eps>0$, 
\[\Prob_{z\sim \calN} \insquare{y\inangle{w, x+\sigma z} > 0} \geq \frac12 + \eps \Longleftrightarrow \frac{y\inangle{w, x}}{\norm{2}{w}} > \sigma \Phi^{-1}\inparen{\frac12 +\eps}.\]
\end{lem}
\noindent \textit{Proof.} The proof follows from the proof of Proposition 3 in \citep*{DBLP:conf/icml/CohenRK19}. We include it bellow for completeness,
    \begin{align*}
        \Prob_{z\sim \calN} \insquare{y\inangle{w, x+\sigma z} > 0} \geq \frac12 + \eps &\Longleftrightarrow \Prob_{z\sim \calN} \insquare{y\inangle{w,\sigma z} > -y\inangle{w, x}} \geq \frac12 + \eps\\
        &\Longleftrightarrow \Prob\insquare{\norm{2}{w}\sigma Z > -y\inangle{w, x}} \geq \frac12 + \eps\\
        &\Longleftrightarrow \Prob\insquare{Z > -y \frac{\inangle{w, x}}{\sigma\norm{2}{w}}} \geq \frac12 + \eps\\
        &\Longleftrightarrow \Prob\insquare{Z < y \frac{\inangle{w, x}}{\sigma\norm{2}{w}}} \geq \frac12 + \eps\\
        &\Longleftrightarrow  y \frac{\inangle{w, x}}{\sigma\norm{2}{w}} > \Phi^{-1}\inparen{\frac12 +\eps}\\
        &\Longleftrightarrow  y \frac{\inangle{w, x}}{\norm{2}{w}} > \sigma\Phi^{-1}\inparen{\frac12 +\eps}. \tag*{\qed}
    \end{align*}
\noindent \textit{Proof of \prettyref{clm:halfspaces-equivalence}}
It follows from \prettyref{lem:perturbation-halfspace} that for any $\gamma > 0$,
    \[
        \min_{w\in \bbR^d} \sum_{t=1}^{T} \max_{z_t\in B(x_t, \gamma)} \ind[y_t \inangle{w,z_t}\leq 0] = \min_{w\in \bbR^d} \sum_{t=1}^{T} \ind\insquare{\frac{y_t\inangle{{w},x_t}}{\norm{\star}{w}}\leq \gamma},
    \]
    and it follows from \prettyref{lem:gaussian-margin-rel} that when $\norm{}{\cdot} = \norm{2}{\cdot}$ and $\sigma, \eps$ satisfy $\gamma=\sigma\Phi^{-1}\inparen{\nicefrac{1}{2} +\nicefrac\eps2}$,
    \[
         \min_{w\in \bbR^d} \sum_{t=1}^{T} \ind\insquare{ y_t \cdot \Ex_{z\sim \calN}[\sign(\inangle{w, x_t+\sigma z})] < \eps}=\min_{w\in \bbR^d} \sum_{t=1}^{T} \ind\insquare{\frac{y_t\inangle{{w},x_t}}{\norm{2}{w}}\leq \gamma}.\tag*{\qed}
    \]
\bibliography{references}

\newpage

\appendix

\section{Smoothed Online Learning}
    \label{sec:smoothed_online}

    In this section, we will discuss a recent line of work in online learning also aimed at circumventing lower bounds corresponding to the adversarial setting: smoothed online learning.
    In smoothed online learning, we posit the existence of a base measure $\mu$ and assume that the distribution of the covariates has bounded density with respect to $\mu$. 
    Formally,

    \begin{defn}[Smoothed Sequences]
        Let $\mathcal{X}$ be a domain and let $\mu$ be a measure on $\mathcal{X}$. 
        A sequence of random variables $x_1, \dots, x_T$ adapted to a filtration $\mathcal{F}_t$ is said to be $\sigma$-smoothed with respect to $\mu$ if for all $t$, the law of $x_t$ conditioned on $\mathcal{F}_{t-1}$, denoted by $\mathcal{D}_t$, satisfies 
        \begin{align}
            \frac{d \mathcal{D}_t }{d \mu }  \leq \frac{1}{\sigma}.
        \end{align}   
    \end{defn}

    The requirement of the uniform bound on the density ratio can be relaxed to weaker notions such as $f$-divergences \cite{DBLP:conf/colt/BlockP23} but we will not delve into these details here.
    Another important consideration in smoothed online learning is the knowledge of the base measure $\mu$.
    Most of the work in this area works under the assumption of a known base measure but recent work has shown that essentially the same results can be recovered with no knowledge of the base measure \cite{DBLP:conf/colt/BlockRS24,DBLP:journals/corr/abs-2410-05124} which we again not focus on here.
    Below we state the regret bound achievable in the smoothed setting.

    \begin{thm}[Smoothed Online Learning] \label{thm:smoothed-online-learning}
        Let $\mathcal{H}$ be a hypothesis class over $\mathcal{X}$.   
        Let $(x_1, y_1)  \dots, (x_T, y_T)$ be a sequence of random variables such that $x_1, \dots , x_T$ that is $\sigma$-smoothed with respect to a measure $\mu$ on $\mathcal{X}$. 
        Then, there exists an algorithm that, for making predictions $\hat{y}_t$ such that 
        \begin{align}
            \Ex \left[ \sum_t \ind\left[ \hat{y}_t \neq y_t \right] - \inf_{h \in  \mathcal{H}    } \sum_{t} \ind \left[ h(x_t) \neq y_t  \right]   \right] \leq \sqrt{T \cdot
        \vc(\mathcal{H}) \cdot  \log(T/ \sigma) }.
        \end{align}
    \end{thm}

    \begin{rem}[Adaptivity] \label{rem:adaptivity}
        A comment regarding the result above is that this allows for the $X_t$ to be dependent on the realizations of $X_{t'}$ for time steps $t' < t $ and not just on the law of $X_{t'}$. 
        Thus, a priori, the result distinguishes having the expectation being outside the infimum versus the expectation inside the supremum. 
        In fact, it turns out that handling this adaptivity is one of the major challenges that \cite{DBLP:journals/jacm/HaghtalabRS24} had to handle. 
        This subtlety will not be the focus of the benchmarks in our paper but might be important in applications when considering which benchmark to use. 
    \end{rem}

    \subsection{Smoothed Online Learning Perspective on Gaussian Smoothed Benchmarks}
    \label{sec:smoothed_online_gaussian}

    We will first use the smoothed online learning framework to derive a regret bound for a benchmark closely related to the benchmark considered in 
    In particular, we will look at the benchmark where we complete with the best classifier but evaluated on a sequence of covariates that have been perturbed with Gaussian noise.
    Formally, consider 
    \begin{align}
        \tilgOPT = \Ex_{{z_1, \ldots, z_T} \sim \mathcal{N}(0,I_d)} \left[ \min_{h \in \mathcal{H}} \sum_{t=1}^{T} \ind\left[ h(x_t + \sigma z_t) \neq y_t \right] \right].
    \end{align}

    In order to compete with this benchmark, the algorithm artificially introduces smoothness by adding Gaussian noise to the covariates.
    In order to state the regret bound, we will need to set up some notation.
    For any set $\mathcal{X}$  and any $a \in \mathbb{R}$, denote by $\mathcal{X}_a = \left\{ x \in \mathbb{R}^d : \inf_{y \in \mathcal{X} }  \norm{2}{x - y} \leq a   \right\} $ the dilation of $\mathcal{X}$ by $a$.  
    For any set $\mathcal{X}$, we will denote by $ \mathrm{Vol} ( \mathcal{X})  $, the Lebesgue volume of $\mathcal{X}$.
    With this notation in place, we can state the following result.

    \begin{cor}
\label{thm:gauss-upperbnd-2}
    For any $\calX\subset \bbR^d$, any $\sigma > 0 $, for any class $\calH\subseteq \left\{ -1,1 \right\}^\calX$, there exists an algorithm guarantees for any sequence $(x_1,y_1),\dots, (x_T, y_T)$, an expected number of mistakes of
    \begin{align}
          \sum_{t=1}^{T}  \ind[ \hat{y}_t \neq y_t ]  -  \tilgOPT \leq  \sqrt{T\cdot \vc(\calH)\ln\inparen{  \frac{\mathrm{Vol}( \mathcal{X}_a    )}{  (\sqrt{ 2 \pi \sigma^2 })^d  }      } }
    \end{align}
    for $a = \sigma \sqrt{d} + 10 \sigma \log(T)$. 
    \end{cor}

    Before we look at the proof of \prettyref{thm:gauss-upperbnd-2}, we compare the benchmark with the $\gOPT$. 
    Recall that the $\gOPT$, defined in \eqref{eqn:OPT_gauss}, is given by
    \begin{align}
        \gOPT \doteq \min_{h \in \calH} \sum_{t=1}^{T} \ind\insquare{ y_t \cdot \Ex_{z\sim \calN(0,I_d)}\insquare{h(x_t+\sigma z)} \leq \eps}.
    \end{align}
    This can be rewritten as 
    \begin{align}
        \gOPT = \min_{h \in \calH} \sum_{t=1}^{T}  \ind\left[ \Pr\left[ h(x_t  + \sigma z ) \neq y_t \right] \leq \frac{1}{2} - \frac{\eps}{2} \right]. 
    \end{align}
    In order to more directly compare the benchmarks, we use Jensen's inequality, to obtain an upper bound
    \begin{align}
        \tilgOPT \leq \tiltilgOPT \doteq  \min_{h \in \mathcal{H}} \sum_{t=1}^{T} \Pr_{z_t \sim \mathcal{N}(0,I_d)} \left[ h(x_t + \sigma z_t) \neq y_t \right].
    \end{align}
    Note that by \prettyref{clm:gaussian}, we always have 
    \begin{align}
        \gOPT \leq 2 \cdot \tiltilgOPT + T\eps
    \end{align}
    for any $\sigma $ and any $ \eps >0 $, while no general reverse inequality holds.
    In this sense, the benchmark $\gOPT$ can be seen as mild refinement of the benchmark $\tiltilgOPT$, though it is best to consider $\tilgOPT$ and $\gOPT$ as imcomparable benchmarks.
    See \prettyref{rem:adaptivity} for further discussion regarding adaptive adversaries which is closely related to the use of Jensen's inequality in this context. 
    Further note that the regret bound in \prettyref{thm:gauss-upperbnd-2} is presented in terms of the volumetric ratio (of a dilation of) while the regret bound in \prettyref{thm:gauss-upperbnd} is phrased in terms of the covering numbers, which while very closely related to the volumetric ratio, can give slightly better bounds in some regimes of parameters \citep[see e.g., Theorem 27.7 in][]{Polyanskiy_Wu_2025}. 

    \begin{claim}
    \label{clm:gaussian}
    $\gOPT \leq 2 \cdot \tiltilgOPT + T\eps.$
    \end{claim}

    \begin{proof}[Proof of \prettyref{clm:gaussian}]
    Observe that for any $h\in \calH$, 
    \begin{align*}
    &\sum_{t=1}^{T}\ind[y_t\cdot\Ex_{z\sim \calN}[h(x_t+\sigma z)] \leq \eps] = \sum_{t=1}^{T} \ind\insquare{ \Prob_{z\sim \calN}\SET{h(x_t+ \sigma z )\neq y_t}  \geq \frac{1}{2} - \frac\eps2} \leq \\
    &\sum_{t=1}^{T} \ind\insquare{ \Prob_{z\sim \calN}\SET{h(x_t+ \sigma z )\neq y_t}  \geq \frac{1}{2} - \frac\eps2} \cdot 2\inparen{\Prob_{z\sim \calN}\SET{h(x_t+ \sigma z )\neq y_t} + \frac\eps2} \leq\\
    & 2\cdot \sum_{t=1}^{T} \Prob_{z\sim \calN}\SET{h(x_t+ \sigma z )\neq y_t} + T\eps. 
    \end{align*}
    \end{proof}

    \subsubsection{Proof of \prettyref{thm:gauss-upperbnd-2}}

    \begin{proof}
        Let $x_t$ be the sequence of covariates. Consider a new sequence of covariates $ \tilde{x}_t = x_t + \sigma z_t $ where $z_t \sim \mathcal{N}(0,\sigma^2 I) $.
        The algorithm adds perturbs $ \tilde{x}_t = x+ \sigma z_t $ where $z_t \sim \mathcal{N}(0,I) $.
        The algorithm then runs a smoothed online learning algorithm from  \prettyref{thm:smoothed-online-learning} with the sequence $ \tilde{x}_t $ and base measure which is uniform over $\mathcal{X} + \sigma \sqrt{d} + 10\sigma \log(T) $.
        In the case, when $ \tilde{x}_t \notin \mathcal{X} + \sigma \sqrt{d} + 10\sigma \log(T)  $ the algorithm predicts a random label.
        Since the probability that $ \tilde{x}_t \notin \mathcal{X} + \sigma  \sqrt{d} + 10\sigma \log(T)  $ is at most $1/T^3$, the overall regret corresponding to these mistakes is at most $1/T$, so we will work with the complement of this event.
        It remains to be shown that the sequence $ \tilde{x}_t $ is $\sigma$-smoothed with respect to the uniform measure which is presented in \prettyref{lem:smooth}.
    \end{proof}

        \begin{lem}
            \label{lem:smooth}
            The sequence $ \tilde{x}_t $, conditioned on the event that all $\tilde{x}_t$ lies in $\mathcal{X} + \gamma \sqrt{d} + 10\sigma \log(T)$, is smoothed with respect to the uniform measure over $\mathcal{X} + \sigma \sqrt{d} + 10\sigma \log(T) $ with smoothness parameter
            \begin{align}
                 \left(\frac{1}{\sigma \sqrt{2\pi} }\right)^d.
            \end{align}
        \end{lem} 
        \begin{proof}
            Note that the density of $\tilde{x}_t$ is given by conditioned on $x_t$ is given by
            \begin{align}
                \frac{ \ind( x_t \in  \mathcal{X} + \sigma \sqrt{d} + 10 \sigma \log T) }{  \Pr\left[ x_t + \sigma Z_t \in \mathcal{X} + \sigma \sqrt{d} + 10 \sigma  \log T  \right] (2\pi)^{d/2} \sigma^d } \int_{\mathcal{X} + \sigma \sqrt{d} + 10 \sigma \log T } \exp\inparen{ - \frac{1}{2\sigma^2} \norm{}{x - \tilde{x}_t} ^2 } dx.
            \end{align}
            As noted above, the probability that $ \tilde{x}_t \notin \mathcal{X} + \sigma \sqrt{d} + 10\sigma \log(T)  $ for all times $t$ is at most $1/T^2$ and thus effects the density ratio only by a constant factor. 
            Note that this density ratio is bounded as required. 
        \end{proof}

    \subsection{ Smoothed Online Learning Perspective on Margin-based Benchmarks}

    Next, we will use the smoothed online learning framework to derive regret bounds analogous to \prettyref{eqn:OPT_input}.
    In order to do this, we introduce some notation.

    \begin{defn} \label{def:growth}
        For a metric space $ \mathcal{X} $, we say a measure is $ \mu $ along with a family of measures $ \mu_{x, \gamma } $ satisfies the $ f(x, \gamma ) $ growth condition if for any $x \in \mathcal{X}$ and $ \gamma > 0$, we have 
    \begin{align} \label{eq:growth_condition}
         \left\|{\frac{d \mu_{x, \gamma} }{ d \mu}} \right\|_{\infty} \leq f(x, \gamma)
     \end{align} 
     where $ \frac{d \mu_{x, \gamma} }{ d \mu}$ is the Radon-Nikodym derivative of the measure $ \mu_{x, \gamma} $ with respect to the measure $ \mu $.
    \end{defn}

     Though this definition seems a bit abstract, we will consider natural settings where this condition is satisfied.
    \newcommand{\muOPT}{ {\mathsf{OPT^{\mu , \gamma } }}}
     Given such a family of measures, we can define the following benchmark
     \begin{align}
        \muOPT = \min_{h \in \calH} \sum_{t=1}^{T} \Ex_{ \tilde{x}_t \sim \mu_{x_t,  \gamma } }  \left[ \ind\left[ h(\tilde{x}_t) \neq y_t \right] \right].
     \end{align}
     For notational simplicity, we just refer $ \mu $ above but note that this depends on the entire system $\mu_{x,\gamma}$ whenever the context is clear.
     We state a result obtaining a regret bound for this benchmark.

     \begin{cor} \label{cor:metric}
        Let $\gamma > 0$ and
        let $\mathcal{X}$ be a metric space equipped with a family of measures $\mu$ and $\{ \mu_{x, \gamma} \}$ satisfying the $f(x, \gamma)$ growth condition (\prettyref{def:growth}). 
        Then, there exists an algorithm guarantees for any sequence $(x_1,y_1),\dots, (x_T, y_T)$, an expected number of mistakes of
        \begin{align}
            \sum_{t=1}^{T}  \ind[ \hat{y}_t \neq y_t ]  -  \muOPT \leq   \sqrt{T\cdot \vc(\calH) \cdot \sup_{x \in \mathcal{X}}   \log(f(x, \gamma ))    }.
        \end{align}
     \end{cor}
     
     As before, we can compare this benchmark with the input-margin benchmark from \eqref{eqn:OPT_input}. 
    Under the assumption that $\mu_{x, \gamma }$ is supported on $B(x,r)$, we have 
    \begin{align} 
        \muOPT \leq \pOPT =  \min_{h \in \calH} \sum_{t=1}^{T} \max_{z_t \in B(x_t, \gamma)} \ind\insquare{h(z_t) \neq y_t}. 
    \end{align}  
      
    Note that it remains to compare the regret bounds. 
    In order to do this, we note that given a minimal covering of the space with $\gamma$ balls, we have a family of measures $\mu_{x, \gamma}$ which samples uniformly from the cover restricted to the ball of radius $\gamma$ at $x$. 
    $\mu$ in this setting is the uniform distribution on the cover. 
    We first note that this family satisfies the growth condition 
    \begin{align}
        \frac{d \mu_{x, \gamma}}{d \mu} \leq 2|\Cover(\calX,\rho,\gamma)|.
    \end{align}
     This is due to the fact that every ball of radius $\gamma$ has at least $1$ point and at most $2$ points (due to packing-covering duality). 
    Thus, using \prettyref{cor:metric}, we get a result analogous to \prettyref{thm:input-margin-upperbnd}. 
    This choice of measures is in fact closely related to the techniques used to prove the regret bound for $\pOPT$, where in fact \prettyref{eqn:cover-opt} show that the algorithm competes with the stronger benchmark.
    
    In addition to this, the condition \prettyref{eq:growth_condition} can be seen as a fractional generalization of covering numbers restricted to scale $\gamma$. 
    This fractional generalization has the advantage of replacing the maximum over a ball with an average over the ball with respect to the family of measure.
    Families of measures that satisfy this condition can be seen to be closely related to doubling measures \citep{doubling} i.e. measures that satisfy
    \begin{align}
            \mu(B(x,2r)) \leq C \mu(B(x,r)).
     \end{align}
     for all $r$ and some $C > 0$.
    It is known that existence of doubling measures is equivalent to having finite doubling dimension, which bound the growth rate of the covering numbers.   
    
     \begin{proof}[Proof of \prettyref{cor:metric}]
        The proof follows by noting that the sequence $ \tilde{x}_t  \sim B(x_t , r ) $ is $\sigma$-smoothed with respect to the measure $ \mu $ for $ \sigma = \sup_{x} f(x, r)   $ and applying the smoothed online learning regret bound from \ref{thm:smoothed-online-learning}.
     \end{proof}

\section{Adversarially Robust Learning with Tolerance}
\label{app:adversarial-learning}

In this section, we show that our result in \prettyref{thm:input-margin-upperbnd} implies a new result for adversarially robust learning with tolerance, a relaxation of adversarially robust learning introduced by \citet*{DBLP:conf/alt/AshtianiPU23}. In this problem, given an i.i.d. sample $S$ drawn from unknown distribuion $D$ over $\calX\times \calY$, the goal is to learn a predictor $\hat{h}: \calX\to \calY$ that minimizes the robust risk:
\[\Risk_\gamma(\hat{h}; D) \doteq \Ex_{(x,y)\sim D} \insquare{\max_{z\in B(x, \gamma)} \ind[{\hat{h}(z)\neq y}]} \leq \inf_{h\in \calH} \Risk_{(1+\alpha)\gamma}(h; D) + \eps,\]
where $B(x,\gamma)$ denotes a ball of radius $\gamma$ centered on $x$ relative to some metric $\rho$ (e.g. $\ell_\infty$) and represents the set of adversarial perturbations that an adversary can choose from at test-time, and $\inf_{h\in \calH} \Risk_{(1+\alpha)\gamma}(h)$ is the relaxed benchmark we compete against parametrized by $\alpha > 0$.

We next state our result for the realizable case where the relaxed benchmark $\inf_{h\in \calH} \Risk_{(1+\alpha)\gamma}(h)=0$. We note that this implies a similar sample complexity bound for the agnostic case with $1/\eps^2$ dependence (as opposed to $1/\eps$ dependence) via a standard reduction from the agnostic case to the realizable case \citep[see e.g., Theorem 6.4 in][]{DBLP:conf/alt/AshtianiPU23}. 

\begin{cor}
\label{cor:robust}
    For any metric space $(\calX, \rho)$, any $\gamma, \alpha > 0$, and any class $\calH\subseteq \calY^\calX$, there exists a learning algorithm $\bbA$ such that for any distribution $D$ over $\calX\times \calY$ where $\inf_{h\in \calH}\Risk_{(1+\alpha)\gamma}(h) = 0$, with probability at least $1-\delta$ over $S\sim D^{m(\eps,\delta)}$,
    \[\Risk_{\gamma} \inparen{\bbA(S)} \leq \eps, \]
    where 
    \[m(\eps, \delta) = O\inparen{\vc(\calH)\ln\inparen{\frac{\abs{\Cover(\calX,\rho,\alpha\gamma)}}{\vc(\calH)}}\frac{1}{\eps}+\frac{1}{\eps}\log\inparen{\frac{1}{\delta}}}.\]
\end{cor}

In comparison, the result of \citet[][Corollary 6.5]{DBLP:conf/alt/AshtianiPU23} crucially requires metric spaces with a doubling metric (and not arbitrary metric spaces), and their stated sample complexity bound depends on the doubling dimension, denoted $d$, of the metric space 
\[m(\eps, \delta) = O\inparen{\vc(\calH) d\ln\inparen{1+\frac{1}{\alpha}}\frac{1}{\eps}+\frac{1}{\eps}\log\inparen{\frac{1}{\delta}} }.\]

\begin{proof}[Proof of \prettyref{cor:robust}]
    For simplicity, we will use the Halving algorithm with the same cover for $\calH$ defined in \prettyref{alg:input-margin} at scale $\alpha\gamma$, and denote the resulting online learning algorithm by $\bbB_{\alpha\gamma}$. By the proof of \prettyref{thm:input-margin-upperbnd}, for any sequence $(z_1, y_1), \dots, (z_T,y_T)$ such that $\OPT^{\alpha\gamma}_{\mathsf{pert}}=\min_{h\in \calH} \sum_{t=1}^{T} \max_{\tilde{z}_t\in B(z_t,\alpha\gamma)} \ind[h(\tilde{z}_t) \neq y_t] =0$, we have the following finite mistake bound guarantee for the predictions of $\bbB_{\alpha\gamma}$, 
    \[\sum_{t=1}^{T} \ind[\hat{y}_t\neq y_t] \leq \vc(\calH)\ln\inparen{\frac{e\abs{\Cover(\calX,\rho,\alpha\gamma)}}{\vc(\calH)}}.\]
    We will use the online learner $\bbB_{\alpha\gamma}$ to construct a stable sample compression scheme for the robust loss $\max_{z\in B(x,\gamma)} \ind[h(z)\neq y]$, where the size of the compression scheme $k = \vc(\calH)\ln\inparen{\frac{e\abs{\Cover(\calX,\rho,\alpha\gamma)}}{\vc(\calH)}}$. By \prettyref{lem:stable-robust-compression} which is due to \citet*[][Lemma 18]{DBLP:conf/colt/MontasserHS21}, this implies the stated sample complexity bound. 

    It remains to describe how to construct the sample compression scheme using the online learner $\bbB_{\alpha\gamma}$. This follows the approach and construction of \citet*[Theorem 1,][]{DBLP:conf/colt/MontasserHS21}, who used it under more general conditions and established bounds based on the Littlestone dimension of $\calH$. We will use a standard online-to-batch conversion scheme. Specifically, given an i.i.d. sample $S=((x_1,y_1),\dots, (x_m, y_m)) \sim D^m$, we cycle a conservative version of the online learner $\bbB_{\alpha\gamma}$ over the examples $(x_i, y_i) \in S$ in order, where each time the learner $\bbB_{\alpha\gamma}$ is not robustly correct on an example $(x_i,y_i)$, i.e., $\exists z_i\in B(x_i,\gamma)$ that $\bbB_{\alpha\gamma}$ labels $-y_i$, we update the online learner $\bbB_{\alpha\gamma}$ by feeding it the example $(z_i y_i)$ and we append the example $(x_i,y_i)$ to the compression sequence. We repeat this until the online learner $\bbB_{\alpha\gamma}$ makes a full pass on $S$ without making any mistakes, i.e., until it robustly and correctly classifies all examples in $S$. Note that because $\inf_{h\in \calH}\Risk_{(1+\alpha)\gamma}(h) = 0$, we are guaranteed that $\min_{h\in\calH} \sum_{i=1}^{m} \max_{z\in B(x_i, (1+\alpha)\gamma)} \ind[h(z)\neq y_i] =0$. Thus, any subsequence $z_1,\dots, z_T$ chosen from $\cup_{i=1}^{m}B(x_i, \gamma)$ will have $\OPT^{\alpha\gamma}_{\mathsf{pert}} =0$, which implies that the online learner $\bbB_{\alpha\gamma}$ will make at most $k$ mistakes from its mistake bound guarantee. Hence, the size of the compression set is at most $k$.
\end{proof}

\begin{lem} [Robust Generalization with Stable Sample Compression, \citet{DBLP:conf/colt/MontasserHS21}]
\label{lem:stable-robust-compression}
Let $(\kappa,\phi)$ be a stable sample compression scheme of size $k$ for $\calH$ with respect to the robust loss $\sup_{z\in B(x,\gamma)} \ind[ h(z)\neq y]$. Then, for any distribution $D$ over $\calX \times \calY$ such that $\inf_{h\in\calH}\Risk_{\gamma}(h;D)=0$, any integer $m > 2k$, and any $\delta\in (0,1)$, with probability at least $1-\delta$ over $S = \{(x_{1},y_{1}),\ldots,(x_{m},y_{m})\}$ iid $D$-distributed random variables,
\[\Risk_{\gamma}(\phi(\kappa(S)); D) \leq \frac{2}{m-2k}\inparen{k\ln(4)+\ln\inparen{\frac{1}{\delta}}}.\]
\end{lem}
\section{Auxiliary Lemmas}

\begin{lem} [See, e.g.~Corollary 2.4 in \citet*{DBLP:books/daglib/0016248}]
    \label{lem:multiplicativeweights}
    Given a finite set of experts $\calF=\SET{f_1,\dots, f_N}$ and an arbitrary sequence of loss functions $\ell_1,\dots, \ell_T: \calF \to [0,1]$, running the Multiplicative Weights algorithm using experts $\calF$ with parameter $\eta > 0$ guarantees
    \[\sum_{t=1}^{T} \Ex_{f\sim P_t} \ell_t(f) \leq \frac{\eta \cdot \min_{f\in \calF} \sum_{t=1}^{T} \ell_t(f)+\ln N}{1-e^{-\eta}}.\]
    In particular, choosing $\eta = \ln\inparen{1+\sqrt{(2\ln N)/(\min_{f\in \calF} \sum_{t=1}^{T} \ell_t(f))}}$ guarantees a regret of
    \[ \sum_{t=1}^{T} \Ex_{f\sim P_t} \ell_t(f) - \min_{f\in \calF} \sum_{t=1}^{T} \ell_t(f) \leq \sqrt{2\cdot \inparen{\min_{f\in \calF} \sum_{t=1}^{T} \ell_t(f)}\cdot \ln N} + \ln N.\]
\end{lem}

\begin{lem}
\label{lem:smoothness} 
Let $f:\mathbb{R}^n\to[-1,1]$ be bounded and define the $\sigma$-smoothed version of $f$ via
\[
   \widehat{f}_\sigma(x)
   \;=\;
   \bigl(f * \mathcal{N}(0,\sigma^2 I)\bigr)(x)
   \;=\;
   \int_{\mathbb{R}^n} f(t)\,\phi_{\sigma}(x - t)\,dt,
\]
where
\[
   \phi_{\sigma}(z)
   \;=\;
   \frac{1}{(2\pi\,\sigma^2)^{n/2}}
   \exp\!\Bigl(-\,\frac{\|z\|^2}{2\,\sigma^2}\Bigr).
\]
Then $\widehat{f}_\sigma$ is $\frac{\sqrt{2/\pi}}{\sigma}$-Lipschitz; that is,
\[
   \|\nabla \widehat{f}_\sigma(x)\|
   \;\le\;
   \frac{\sqrt{2/\pi}}{\sigma}
   \quad\text{for all }x\in\mathbb{R}^n.
\]
\end{lem}

\begin{proof}
We adapt the proof of \citet*[][Lemma 1]{DBLP:conf/nips/SalmanLRZZBY19} to handle arbitrary $\sigma > 0$. Because $f$ is bounded by $1$ in absolute value, it suffices to show that
\[
  \sup_{\|u\|=1} \bigl|\,u \cdot \nabla \widehat{f}_\sigma(x)\bigr|
  \;\le\;
  \frac{\sqrt{2/\pi}}{\sigma}.
\]
By differentiating under the integral, we get
\[
  u \cdot \nabla \widehat{f}_\sigma(x)
  \;=\;
  \int_{\mathbb{R}^n} f(t)\,
  \Bigl( u \cdot \nabla_x\,\phi_{\sigma}(x - t) \Bigr)\,dt.
\]
Taking absolute values and using $|f(t)|\le 1$, we obtain
\[
  \bigl|\,u \cdot \nabla \widehat{f}_\sigma(x)\bigr|
  \;\le\;
  \int_{\mathbb{R}^n} 
      \bigl|\,u \cdot \nabla_x\,\phi_{\sigma}(x - t)\bigr|
  \,dt.
\]
Set $z = x - t$.  Then $\phi_\sigma(z) 
  = \frac{1}{(2\pi\,\sigma^2)^{n/2}} \exp\!\bigl(-\|z\|^2/(2\sigma^2)\bigr)$,
and direct computation shows
\[
   \nabla_z\,\phi_{\sigma}(z)
   \;=\;
   -\,\frac{1}{\sigma^2}\,z\,\phi_{\sigma}(z).
\]
Hence,
\[
  \bigl|\,u \cdot \nabla_x\,\phi_{\sigma}(z)\bigr|
  \;=\;
  \bigl|\,u \cdot \nabla_z\,\phi_{\sigma}(z)\bigr|
  \;=\;
  \frac{1}{\sigma^2}
  \,\bigl|\,\langle z,u\rangle\bigr|\,
  \phi_{\sigma}(z).
\]
Therefore
\[
  \int_{\mathbb{R}^n}
      \bigl|\,u \cdot \nabla_x\,\phi_{\sigma}(z)\bigr|
  \,dz
  \;=\;
  \frac{1}{\sigma^2}
  \int_{\mathbb{R}^n}
    \bigl|\,\langle z,u\rangle\bigr|\,
    \phi_{\sigma}(z)
  \,dz.
\]
Observe that under the kernel $\phi_{\sigma}(z)$, the random vector $z$ is distributed as $\mathcal{N}(0,\sigma^2 I)$.  Since $u$ is a unit vector, $\langle z,u\rangle$ is distributed as $\mathcal{N}(0,\sigma^2)$.  It follows that
\[
  \int_{\mathbb{R}^n}
    \bigl|\,\langle z,u\rangle\bigr|\,
    \phi_{\sigma}(z)\,dz
  \;=\;
  \sigma\,\sqrt{\frac{2}{\pi}}.
\]
Combining these,
\[
  \int_{\mathbb{R}^n}
      \bigl|\,u \cdot \nabla_x\,\phi_{\sigma}(z)\bigr|
  \,dz
  \;=\;
  \frac{1}{\sigma^2} \cdot \Bigl(\sigma\,\sqrt{\tfrac{2}{\pi}}\Bigr)
  \;=\;
  \frac{\sqrt{2/\pi}}{\sigma}.
\]
Hence 
\[
  \bigl|\,u \cdot \nabla \widehat{f}_\sigma(x)\bigr|
  \;\le\;
  \frac{\sqrt{2/\pi}}{\sigma}.
\]
Since $u$ was an arbitrary unit vector, we conclude that
\[
  \|\nabla \widehat{f}_\sigma(x)\|
  \;\le\;
  \frac{\sqrt{2/\pi}}{\sigma},
\]
finishing the proof.
\end{proof}

\end{document}